\documentclass[letterpaper]{article} 
\usepackage{aaai23}  
\usepackage{times}  
\usepackage{helvet}  
\usepackage{courier}  
\usepackage[hyphens]{url}  
\usepackage{graphicx} 
\usepackage{natbib}  
\usepackage{caption} 
\usepackage{algorithm}
\usepackage{algorithmic}

\usepackage{newfloat}
\usepackage{listings}
\DeclareCaptionStyle{ruled}{labelfont=normalfont,labelsep=colon,strut=off} 
\floatstyle{ruled}
\newfloat{listing}{tb}{lst}{}
\floatname{listing}{Listing}
\usepackage[utf8]{inputenc} 
\usepackage{booktabs}       
\usepackage{amsfonts}       
\usepackage{nicefrac}       
\usepackage{microtype}      
\usepackage{xcolor}         

\usepackage{amsmath}
\usepackage{amssymb}
\usepackage{mathtools}
\usepackage{amsthm}
\usepackage{paralist}
\usepackage[capitalize,noabbrev]{cleveref}
\usepackage{soul}
\usepackage[inline]{enumitem}
\usepackage{thmtools} 
\usepackage{thm-restate}
\usepackage{dsfont}

\theoremstyle{plain}
\newtheorem{theorem}{Theorem}[section]
\newtheorem{proposition}[theorem]{Proposition}

\theoremstyle{definition}

\theoremstyle{remark}

\usepackage[textsize=tiny]{todonotes}
\usepackage{svg}
\usepackage{subcaption}
\graphicspath{ {./images/} }
\usepackage{amsmath,amsfonts,bm}

\def\eqref#1{equation~\ref{#1}}

\def\1{\bm{1}}

\def\vf{{\bm{f}}}

\def\vu{{\bm{u}}}

\def\vx{{\bm{x}}}

\def\mA{{\bm{A}}}

\def\mD{{\bm{D}}}

\def\mH{{\bm{H}}}
\def\mI{{\bm{I}}}

\def\mL{{\bm{L}}}

\def\mR{{\bm{R}}}
\def\mS{{\bm{S}}}
\def\mT{{\bm{T}}}
\def\mU{{\bm{U}}}

\def\mX{{\bm{X}}}
\def\mY{{\bm{Y}}}

\DeclareMathAlphabet{\mathsfit}{\encodingdefault}{\sfdefault}{m}{sl}
\SetMathAlphabet{\mathsfit}{bold}{\encodingdefault}{\sfdefault}{bx}{n}

\def\gG{{\mathcal{G}}}

\DeclareMathOperator*{\argmin}{arg\,min}

\renewcommand\vec{\mathbf}

\newcommand*{\cora}{{\sc Cora}}

\newcommand*{\cornell}{{\sc Cornell}}
\newcommand*{\texas}{{\sc Texas}}
\newcommand*{\wisconsin}{{\sc Wisconsin}}
\newcommand*{\chameleon}{{\sc Chameleon}}
\newcommand*{\squirrel}{{\sc Squirrel}}
\newcommand*{\actor}{{\sc Actor}}
\newcommand*{\euroairport}{{\sc Europe Airport}}

\title{Restructuring Graphs for Higher Homophily via Adaptive Spectral Clustering}

\author{
    Shouheng Li,\textsuperscript{\rm 1,\rm 3}
    Dongwoo Kim,\textsuperscript{\rm 2}
    Qing Wang\textsuperscript{\rm 1}
}
\affiliations{
    \textsuperscript{\rm 1} School of Computing, Australian National University, Canberra, Australia\\
    \textsuperscript{\rm 2} CSE \& GSAI, POSTECH, Pohang, South Korea\\
    \textsuperscript{\rm 3} Data61, CSIRO, Canberra, Australia\\

    shouheng.li@anu.edu.au, dongwoo.kim@postech.ac.kr, qing.wang@anu.edu.au
}

\begin{document}

\maketitle

\begin{abstract}

While a growing body of literature has been studying new Graph Neural Networks (GNNs) that work on both homophilic and heterophilic graphs, little has been done on adapting classical GNNs to less-homophilic graphs. Although the ability to handle less-homophilic graphs is restricted, classical GNNs still stand out in several nice properties such as efficiency, simplicity, and explainability. In this work, we propose a novel graph restructuring method that can be integrated into any type of GNNs, including classical GNNs, to leverage the benefits of existing GNNs while alleviating their limitations. Our contribution is threefold:
\begin{enumerate*}[label=\emph{\alph*})]
\item learning the weight of \emph{pseudo-eigenvectors} for an adaptive spectral clustering that aligns well with known node labels, 
\item proposing a new density-aware homophilic metric that is robust to label imbalance, and
\item reconstructing the adjacency matrix based on the result of adaptive spectral clustering to maximize homophilic scores. The experimental results show that our graph restructuring method can significantly boost the performance of six classical GNNs by an average of 25\% on less-homophilic graphs. The boosted performance is comparable to state-of-the-art methods.\footnote{The extended version is available at \url{https://arxiv.org/abs/2206.02386}. The code is available at \url{https://github.com/seanli3/graph_restructure}.}
\end{enumerate*}

\end{abstract}

\section{Introduction}
\label{sec:intro}

\begin{figure}[t!]
\centering
\begin{subfigure}{.45\columnwidth}
\includegraphics[clip,width=\textwidth,bb=0 0 289 286]{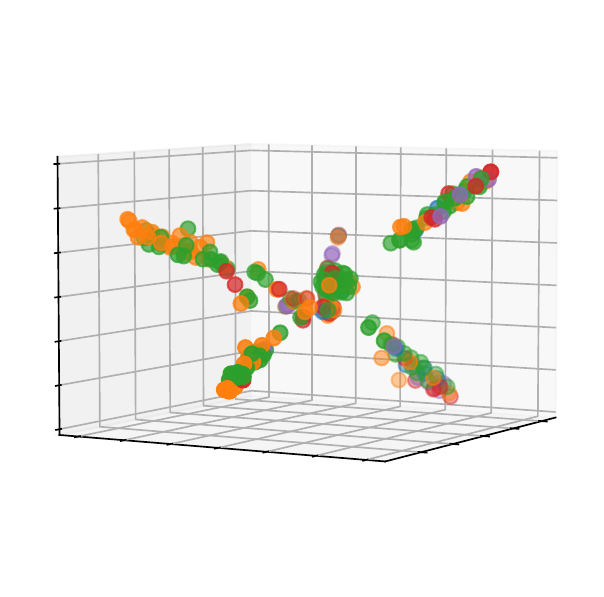}
\subcaption{\wisconsin{}: visualized using T-SNE with the leading $5$ eigenvectors}
\label{subfig:wisconsin_low_sc}
\end{subfigure}
\hfill
\begin{subfigure}{.45\columnwidth}
\includegraphics[clip,width=\textwidth]{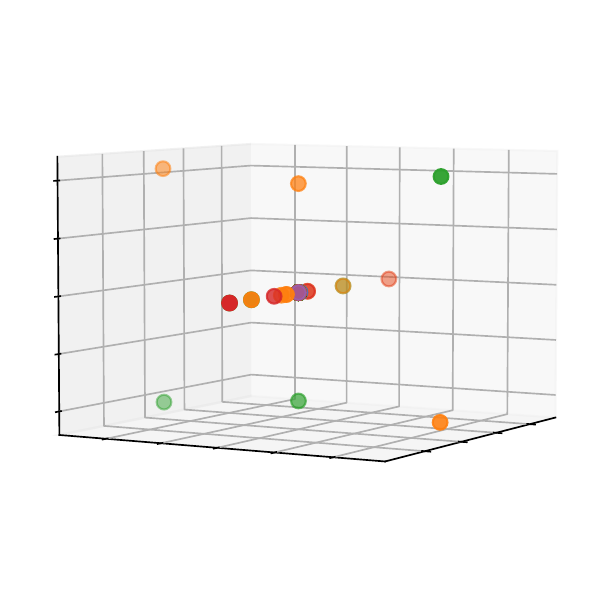}
\subcaption{\wisconsin{}: visualized using 22th, 44th and 206th eigenvectors}
\label{subfig:wisconsin_sc_manual}
\end{subfigure}
\\
\begin{subfigure}{.45\columnwidth}
\includegraphics[clip,width=\textwidth]{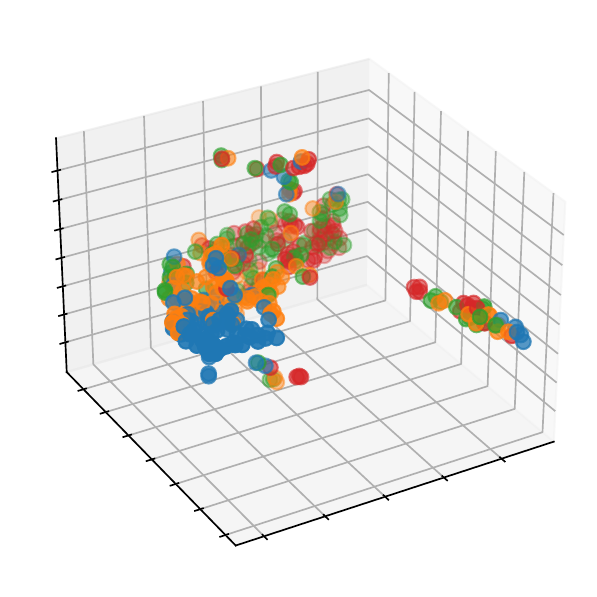}
\subcaption{\euroairport{}: visualized using T-SNE with the leading $5$ eigenvectors}
\label{subfig:europe_low_sc}
\end{subfigure}
\hfill
\begin{subfigure}{.45\columnwidth}
\includegraphics[clip,width=\textwidth]{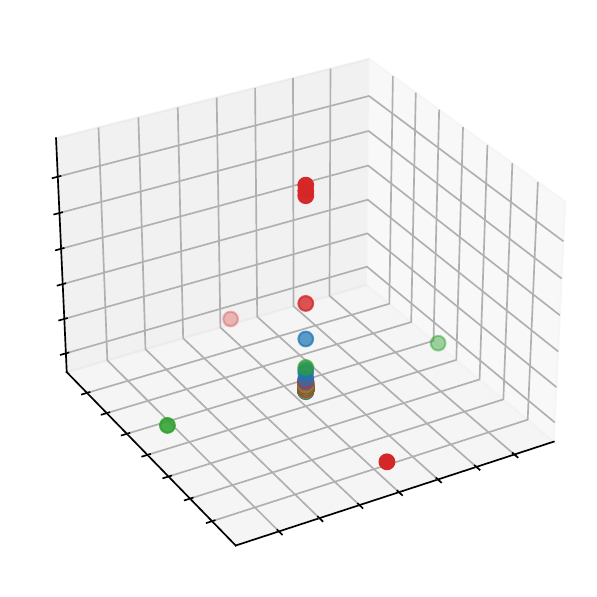}
\subcaption{\euroairport{}: visualized using the 366th, 382th and 3rd eigenvectors}
\label{subfig:europe_sc_manual}
\end{subfigure}
\caption{Node clusters using different eigenvector choices on \wisconsin{} and \euroairport{}. Colours represent node labels. Coordinates in \ref{subfig:wisconsin_low_sc} and \ref{subfig:europe_low_sc} are computed using three-dimensional T-SNE. In \ref{subfig:wisconsin_sc_manual} and \ref{subfig:europe_sc_manual} many nodes are overlapped so they appear to have fewer nodes than they actually have.}
\label{fig:tsne_embed}
\end{figure}


Graph Neural Networks (GNNs) were originally inspired under the homophilic assumption - nodes of the same label are more likely to be connected than nodes of different labels.
Recent studies have revealed the limitations of this homophilic assumption when applying GNNs on less-homophilic or heterophilic graphs~\citep{Pei2020}.
Since then, a number of approaches have been proposed with a focus on developing deep learning architectures for heterophilic graphs~\citep{zhu2021graph, kim21howto, chien21, Li2021-uj, Bo2021-tm, Lim2021-qz}.
However, little work has been done on adapting existing GNNs to less-homophilic graphs. Although existing GNNs such as GCN~\citep{DBLP:conf/iclr/KipfW17}, ChevNet~\citep{DBLP:conf/nips/DefferrardBV16}, and GAT~\citep{DBLP:conf/iclr/VelickovicCCRLB18} are lacking the ability to work with less-homophilic graphs, they still stand out in several nice properties such as efficiency~\citep{Zeng2019-nz}, simplicity~\citep{DBLP:conf/icml/WuSZFYW19}, and explainability~\citep{Ying2019-ws}. 
Our work aims to develop a graph restructuring method that can restructure a graph to leverage the benefit of prevalent homophilic GNNs. Inspired by~\citet{Klicpera2019-ut}, we extend the restructuring approach to heterophilic graphs and boost the performance of GNNs that do not work well on less-homophilic graphs.

\paragraph{Discrepancy between SC and labels}
Spectral clustering (SC) aims to cluster nodes such that the edges between different clusters have low weights and the edges within a cluster have high weights. Such clusters are likely to align with node labels where a graph is homophilic, i.e., nodes with the same labels are likely to connect closely. Nonetheless, this may not hold for less-homophilic graphs. Figure \ref{fig:tsne_embed} visualizes the nodes in \wisconsin{}~\citep{Pei2020} and \euroairport{}~\citep{Ribeiro2017} datasets based on the eigenvectors corresponding to the five smallest eigenvalues. As shown in Figure \ref{subfig:wisconsin_low_sc} and \ref{subfig:europe_low_sc}, the labels are not aligned with the node clusters. However, if we choose the eigenvectors carefully, nodes clusters can align better with their labels, as evidenced by Figure \ref{subfig:wisconsin_sc_manual} and \ref{subfig:europe_sc_manual}, which are visualized using manually chosen eigenvectors. 

This observation shows that the eigenvectors corresponding to the leading eigenvalues do not always align well with the node labels. Particularly, in a heterophilic graph, two adjacent nodes are unlikely to have the same label, which is in contradiction with the smoothness properties of leading eigenvectors. However, when we choose eigenvectors appropriately, the correlation between the similarity of spectral features and node labels increases.
To generalize this observation, we propose an adaptive spectral clustering algorithm that, 
\begin{enumerate*}[label=\emph{\alph*})]
\item divides the Laplacian spectrum into even slices, each slice corresponding to an embedding matrix called \emph{pseudo-eigenvector};
\item learns the weights of \emph{pseudo-eigenvectors} from existing node labels;
\item restructures the graph according to node embedding distance to maximize homophily.
\end{enumerate*}
To measure the homophilic level of graphs, we further introduce a new density-aware metric that is robust to label imbalance and label numbers. Our experimental results show that the performances of node-level prediction tasks with restructured graphs are greatly improved on classical GNNs.





\section{Background}
\label{sec:pre}

\paragraph{Spectral filtering.}
Let $\gG=(V,E,\mA, \mX)$ be an undirected graph with $N$ nodes, where $V$, $E$, and $\mA \in \{0, 1\}^{N\times N}$ are the node set, edge set, and adjacency matrix of $\gG$, respectively, and $\mX \in \mathbb{R}^{N\times F}$ is the node feature matrix. The normalized Laplacian matrix of $\gG$ is defined as $\mL = \mI - \mD^{-1/2}\mA\mD^{-1/2}$, where $\mI$ is the identity matrix with $N$ diagonal entries and $\mD \in {\mathbb R}^{N\times N}$ is the diagonal degree matrix of $\gG$. 
In spectral graph theory, the eigenvalues $\Lambda = \text{diag}(\lambda_1,...,\lambda_N)$ and eigenvectors $\mU$ of $\mL = \mU\Lambda \mU^H$ are known as the graph's spectrum and spectral basis, respectively, where $\mU^H$ is the Hermitian transpose of $\mU$. The graph Fourier transform takes the form of $\hat{\mX} = \mU^H\mX$ and its inverse is $\mX = \mU\hat{\mX}$.

It is known that the Laplacian spectrum and spectral basis carry important information on the connectivity of a graph~\citep{DBLP:journals/spm/ShumanNFOV13}. Lower frequencies correspond to global and smooth information on a graph, while higher frequencies correspond to local information, discontinuities and possible noise~\citep{DBLP:journals/spm/ShumanNFOV13}. One can apply a spectral filter and use graph Fourier transform to manipulate signals on a graph in various ways, such as smoothing and denoising~\citep{DBLP:conf/globalsip/SchaubS18}, abnormally detection~\citep{5967745} and clustering~\citep{8462239}.
Spectral convolution on graphs is defined as the multiplication of a signal $\vx$ with a filter $g(\Lambda)$ in the Fourier domain, i.e.,
\begin{equation}\label{eqn:spectral_filter}
g(\mL)\vec{x} = g(\mU\Lambda \mU^H)\vx = \mU g(\Lambda)\mU^H\vx = \mU g(\Lambda)\hat{\vx}.
\end{equation}

\paragraph{Spectral Clustering (SC) as low-pass filtering.}\label{para:sc_low} 
SC is a well-known method for clustering nodes in the spectral domain. A simplified SC algorithm is described in the Appendix.
Classical SC can be interpreted as low-pass filtering on a one-hot node signal $\delta_i \in \mathbb{R}^N$, where the $i$-th element of $\delta_i$ is $1$, i.e. $\delta_i(i) =1$, and $0$ elsewhere, for each node $i$. Filtering $\delta_i$ in the graph Fourier domain can be expressed as 
\begin{equation}
\label{eqn:sc_lowpass}
    \vf_{i} = g_{\lambda_L}(\Lambda)\mU^H\delta_i,
\end{equation}
where $g_{\lambda_L}$ is the low-pass filter that filter out components whose frequencies are greater than $\lambda_L$ as 
\begin{align}
\label{eqn:low_pass}
    g_{\lambda_L}(\lambda) =
    \begin{cases}
      1 & \text{if $\lambda \leq \lambda_L$;}\\
      0 & \text{otherwise}.
    \end{cases}
\end{align}

As pointed out by \citet{DBLP:conf/icassp/TremblayPBGV16,DBLP:conf/nips/RamasamyM15}, the node distance $\lVert\vf_i-\vf_j\rVert^2$ used in SC can be approximated by filtering random node features with $g_{\lambda_L}$. Consider a random node feature matrix $\mR = [\vec{r}_1|\vec{r}_2|...|\vec{r}_\eta]\in \mathbb{R}^{N \times \eta}$ consisting of $\eta$ random features $\vec{r}_i \in\mathbb{R}^{N}$ i.i.d sampled form the normal distribution with zero mean and $1/\eta$ variance, let $\mH_{\lambda_L} = \mU g_{\lambda_L}(\Lambda)\mU^H$, 
we can define 
\begin{equation}
\label{eqn:sc_feature}
\tilde{\vf}_{i} = (\mH_{\lambda_L}\mR)^H\delta_i.
\end{equation}
 $\mR$ is a random Gaussian matrix of zero mean and $\mU$ is orthonormal. We apply the Johnson-Lindenstrauss lemma (see Appendex) 
 to obtain the following error bounds.

\begin{proposition}[\citet{DBLP:conf/icassp/TremblayPBGV16}]
\label{proposition:tremblay}
Let $\epsilon,\beta > 0$ be given. If $\eta$ is larger than:
\begin{align}
    \eta_0 = \frac{4+2\beta}{\epsilon^2/2 - \epsilon^3/3}\log{N},
\end{align}
then with probability at least $1-N^{-\beta}$, we have: $\forall (i,j)\in [1, N]^2$,
\begin{align}
\label{eqn:error_bounds}
    (1-\epsilon)\lVert \vf_i - \vf_j \rVert^2 \leq \lVert \tilde{\vf}_i - \tilde{\vf}_j \rVert^2 \leq (1+\epsilon)\lVert\vf_i - \vf_j\rVert^2.
\end{align}
\end{proposition}

Hence, $\lVert\tilde{\vf}_{i} - \tilde{\vf}_{j}\rVert$ is a close estimation of the Euclidean distance $\lVert\vf_{i} - \vf_{j}\rVert$.
Note that the approximation error bounds in Equation~\ref{eqn:error_bounds} also hold for any band-pass filter $\bar{g}_{e_1,e_2}$: 
\begin{align}
\label{eqn:band_pass}
    \bar{g}_{e_1,e_2}(\lambda) =
    \begin{cases}
      1 & \text{if $e_1 < \lambda \leq e_2$}\\
      0 & \text{otherwise}.
    \end{cases}
\end{align}
where $0 \leq e_1 \leq e_2 \leq 2$ since the spectrum of a normalized Laplacian matrix has the range $[0,2]$. 
This fact is especially useful when the band pass $\bar{g}_{e_1,e_2}$ can be expressed by a specific functional form such as polynomials, since $\mH_{e_1,e2}=\mU \bar{g}_{e_1,e2}(\Lambda)\mU^H$ can be computed without eigen-decomposition.



\section{Adaptive Spectral Clustering}
\label{sec:method}

\begin{figure*}[t]
\begin{subfigure}{.7\textwidth}
\includegraphics[clip,width=\columnwidth]{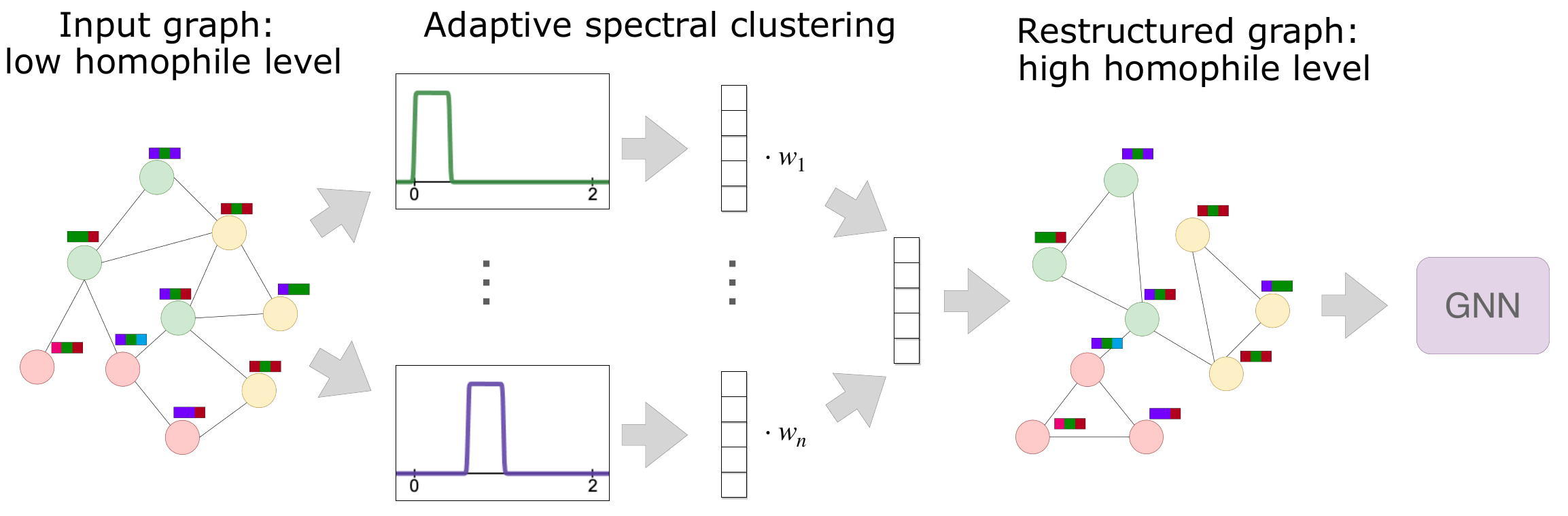}
\subcaption{Graph restructuring workflow using adaptive spectral clustering}
\label{fig:workflow}
\end{subfigure}
\hfill
\begin{subfigure}{.25\textwidth}
\includegraphics[clip,width=\columnwidth]{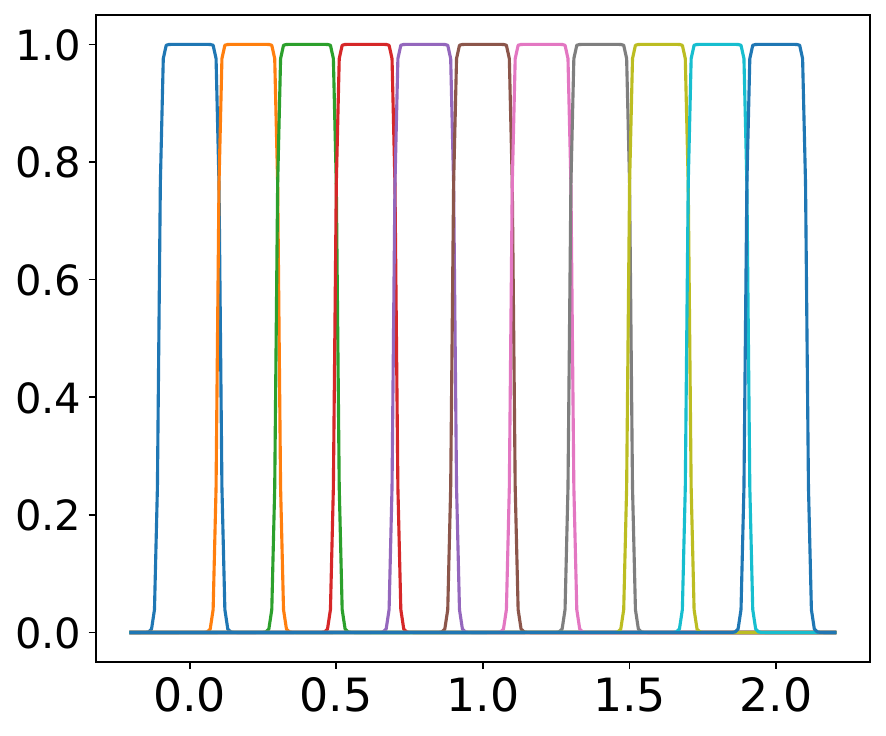}
\subcaption{A set of slicers}
\label{fig:slicers}
\end{subfigure}
\caption{Adaptive spectral clustering using spectrum slicers. In \ref{fig:workflow} $w_1, ..., w_n$ are adaptive scalar parameters. In \ref{fig:slicers}, the spectrum range $[0,2]$ is sliced into equal-width pseudo-eigenvalues by a set of slicers with $s=20$.}
\end{figure*}


Now we propose an adaptive method for spectral clustering, which aligns the clustering structure with node labels by learning the underlying frequency patterns. This empowers us to restructure a graph to improve graph homophily while preserving the original graph structures as much as possible.




\subsection{Learning eigenvector coefficients}
Let $\vf_i^\mathcal{Z}$ be the representation of a node $i$ obtained from an arbitrary set $\mathcal{Z} \subseteq [N]$ of eigenvectors in place of the ones with the leading eigenvalues. We cast the problem of finding eigenvectors to align with the known node labels into a minimization problem of computing the distance between the representations of nodes $i$ and $j$ when their labels are the same. Let $d(\cdot)$ be a distance metric between two node representations. Then, our objective is formalized as
\begin{align}
    \label{eqn:objective-raw}
    \argmin_\mathcal{Z} \sum_{i, j \in V_Y} \ell(d(\vf_i^\mathcal{Z}, \vf_j^\mathcal{Z}), \mathds{1}(y_i, y_j)),
\end{align}
where $V_Y$ is a collection of nodes whose labels are available, $\mathds{1}$ is an indicator function, $y_i$ is the label of node $i$, and $\ell$ is a loss function. 
The loss function penalizes the objective when two nodes of the same label are far from each other, as well as when two nodes of different labels are close.
In this work, we use Euclidean distance and a contrastive loss function.

Solving the above objective requires iterating over all possible combinations of eigenvectors, which is infeasible in general. It also requires performing expensive eigendecomposition with $O(N^3)$ time complexity. In addition, classical SC does not consider node features when computing node distance. To address these challenges, we introduce two key ideas to generalize SC in the following.

\subsection{Eigendecomposition-free SC}
As explained in Section \ref{sec:pre}, $\lVert\vf_i - \vf_j\rVert$ can be approximated by a filtering operation under the Johnson-Lindenstrauss lemma. The same holds true for the generalized case $\vf_i^\mathcal{Z}$. 
However, the operation still requires expensive eigendecomposition as Equation~\ref{eqn:band_pass} takes eigenvalues explicitly. To mitigate this issue, we propose to use a series of rectangular functions, each serving as a band-pass filter that ``slices" the Laplacian spectrum into a finite set of equal-length and equal-magnitude ranges. Each filter takes the same form of Equation \ref{eqn:band_pass}, but is relaxed on the continuous domain. This formulation comes with two major advantages. Firstly, rectangular functions can be efficiently approximated with polynomial rational functions, thus bypassing expensive eigendecomposition. Secondly, each $g_{e_1,e_2}$ groups frequencies within the range $(e_1, e_2]$ to form a ``coarsened" \emph{pseudo-eigenvalue}. Because nearby eigenvalues capture similar structural information, the ``coarsening" operation reduces the size of representations, while still largely preserving the representation power.

\paragraph{Spectrum slicers.}

We approximate the band-pass rectangular filters in Equation \ref{eqn:band_pass} using a rational function  
\begin{equation}
\label{eqn:slicer}
    \hat{g}_{s,a} = \frac{1}{s^{2m}}\left(\left(\frac{\lambda-a}{2+\hat{\epsilon}} \right)^{2m}+\frac{1}{s^{2m}}\right)^{-1}
\end{equation}
where $s\geq2$ is a parameter that controls the width of the passing window on the spectrum, $a\in[0,2]$ is a parameter that controls the horizontal center of the function, $m$ is the approximation order. Figure \ref{fig:slicers} shows an example of these functions. With properly chosen $s$ and $a$, the Laplacian spectrum can be evenly sliced into chunks of range $(\lambda_{i}, \lambda_{i+1})$. Each chunk is a \emph{pseudo-eigenvalue} that umbrellas eigenvalues within the range. Substituting $g(\lambda)$ in Equation \ref{eqn:spectral_filter} with $\hat{g}_{s,a}$, 
the spectral filtering operation becomes
\begin{equation}
\label{eqn:slicer_matrix}
\mU \hat{g}_{s,a}(\Lambda)\mU^H\vx =\frac{1}{s^{2m}}\left(\left(\frac{\mL-a\mI}{2+\hat{\epsilon}}\right)^{2m}+\frac{\mI}{s^{2m}}\right)^{-1} 
\notag
\end{equation}
where $\mU\Lambda\mU^H = \mL$. 
An important property of Equation \ref{eqn:slicer_matrix} is that the matrix inversion can be computed via truncated Neumann series~\citep{Wu2013-bt}. This can bring the computation cost of $O(N^3)$ down to $O(pN^2)$.
\begin{restatable}[]{lemma}{newmannlemma}
\label{lemma:neumann}
For all $\hat{\epsilon} > \frac{2s^{2m}}{s^{2m}-1}-2$, the inverse of $\mT = \left(\frac{\mL-a\mI}{2+\hat{\epsilon}}\right)^{2m}+\frac{\mI}{s^{2m}}$ can be expressed by a Neumann series with guaranteed convergence (A proof is given in Appendix).
\end{restatable}

\subsection{SC with node features}
Traditionally, SC does not use node features. However, independent from graph structure, node features can provide extra information that is valuable for a clustering task~\citep{Bianchi2019-qj}. We therefore incorporate node features into the SC filtering operation by concatenating it with the random signal before the filtering operation
\begin{align}
    \label{eqn:filtered}
    \bm{\Gamma}_{s,a} = \hat{g}_{s,a}(\mL)(\mR \frown \mX),
\end{align}
where $\frown$ is a column-wise concatenation operation. $\bm{\Gamma}_{s,a}$ has the shape of $N \times (P+F)$ and is sometimes referred as "dictionary"~\citep{Thanou2014-xz}. 
When using a series of $g_{s,a}$, we have 
\begin{align}
    \bm{\Gamma} = (\bm{\Gamma}_{s_1,a_1} \frown \bm{\Gamma}_{s_2,a_2} \frown ...).
\end{align}
Let $P'$ be the dimension of embeddings and $\Theta$ represent a weight matrix or a feed-forward neural network. The concatenated dictionary is then fed into a learnable function to obtain a node embedding matrix $\mH \in \mathbb{R}^{N \times P'}$ as
\begin{align}
\label{eqn:node_embedding}
    \mH = \Theta(\bm{\Gamma}).
\end{align}
Our objective in Equation \ref{eqn:objective-raw} can then be realized as
\begin{align}
\label{eqn:objective}
    \mathcal{L}(\Theta) = \sum_{\substack{i,j \in V_Y \\ k \in \mathcal{N}_Y(i)\\ y_i=y_j}} \left[ ||\mH_{i\cdot} - \mH_{j\cdot} ||^2  - || \mH_{i\cdot} - \mH_{k\cdot} ||^2 + \epsilon  \right]_{+}
\end{align}
where $\mathcal{N}_Y(k)$ is a set of negative samples whose labels are different from node $i$, i.e. $y_i\neq y_k$. The negative samples can be obtained by randomly sampling a fixed number of nodes with known labels. The intuition is if the labels of nodes $i$ and $j$ are the same, then the distance between the two nodes needs to be less than the distance between $i$ and $k$, to minimize the loss. $[a]_+ = \max(a,0)$ and $\epsilon$ is a scalar offset between distances of intra-class and inter-class pairs. By minimizing the objective w.r.t the weight matrix $\Theta$, we can learn the weight of each band that aligns the best with the given labels.
Ablation study and spectral expressive power of this method is discussed separately in Appendix.

\subsection{Restructure graphs to maximize homophily}
\label{subsec:restructure}

After training, we obtain a distance matrix $\mD'$ where $\mD'_{ij}=\lVert\mH_{i\cdot}-\mH_{j\cdot} \rVert$. An intuitive way to reconstruct a graph is to start with a disconnected graph, and gradually add edges between nodes with the smallest distances, until the added edges do not increase homophily. We adopt the same approach as \citet{Klicpera2019-ut}, to add edges between node pairs of the $K$ smallest distance for better control of sparsity. Another possible way is to apply a threshold on $\mD'$ and entries above the threshold are kept as edges. Specifically, $\hat{\mA}$ is the new adjacency matrix whose entries are defined as
\begin{align}
    \hat{\mA}_{ij} = \begin{cases}
                    1, & \text{if } (i,j) \in \operatorname{topK^{-1}}(\mS)\\
                    0, & \text{otherwise},
                \end{cases}
\end{align}
where $\operatorname{topK^{-1}}$ returns node pairs of the $k$ smallest entries in $\mD'$. A simplified workflow is illustrated in Figure \ref{fig:workflow}. A detailed algorithm can be found in Appendix. 

\subsection{Complexity analysis}
\label{subsec:complexity}
The most expensive operation of our method is the matrix inversion in Equation \ref{eqn:slicer_matrix} which has the time complexity of $O(pN^2)$. A small $p\leq4$ is sufficient because the Neumann series is a geometric sum so exponential acceleration tricks can be applied. Equation \ref{eqn:slicer_matrix} is a close approximation to a rectangular function that well illustrates the spectrum slicing idea. In practice, it can be replaced by other slicer functions that do not require matrix inversion, such as a quadratic function $1-\left(s\mL-a\right)^{2}$, to further reduce cost. It is also worth noting that the matrix inversion and multiplication only need to be computed once and can be pre-computed offline as suggested by~\citep{RossiSIGN2020}. The training step can be mini-batched easily. We randomly sample $8 {-} 64$ negative samples per node so the cost of computing Equation \ref{eqn:objective} is low.

\section{A New Homophily Measure}
\label{sec:homo}

\begin{figure*}[t!]
         \centering
     \begin{subfigure}[b]{0.24\columnwidth}
         \centering
         \includegraphics[width=0.8\textwidth]{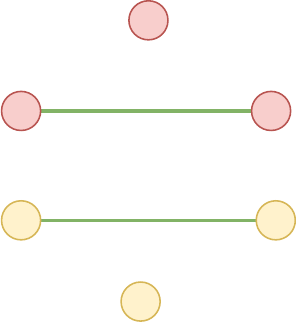}
         \caption{\\$h_{\text{edge}}=1$ \\ $h_{\text{node}}=1$ \\ $h_{\text{norm}}=1$ \\ $h_{\text{den}}=0.58$}
         \label{fig:homo1}
     \end{subfigure}
     \hfill
     \begin{subfigure}[b]{0.24\columnwidth}
         \centering
         \includegraphics[width=0.8\textwidth]{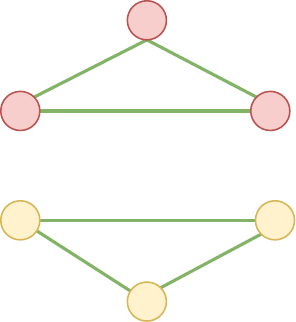}
         \caption{\\$h_{\text{edge}}=1$ \\ $h_{\text{node}}=1$ \\ $h_{\text{norm}}=1$ \\ $h_{\text{den}}=0.75$}
         \label{fig:homo2}
     \end{subfigure}
     \hfill
     \begin{subfigure}[b]{0.24\columnwidth}
         \centering
         \includegraphics[width=0.8\textwidth]{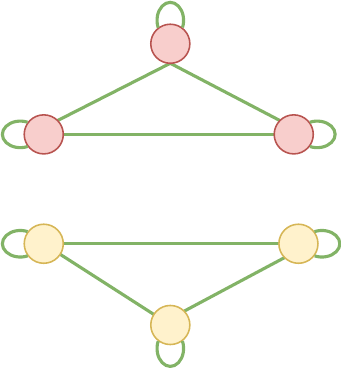}
         \caption{\\$h_{\text{edge}}=1$ \\ $h_{\text{node}}=1$ \\ $h_{\text{norm}}=1$ \\ $h_{\text{den}}=1$}
         \label{fig:homo3}
     \end{subfigure}
     \hfill
     \begin{subfigure}[b]{0.24\columnwidth}
         \centering
         \includegraphics[width=0.8\textwidth]{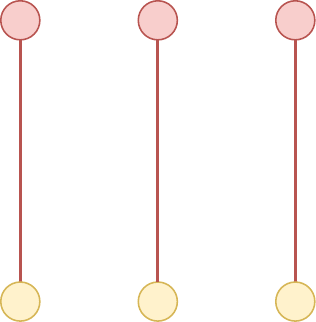}
         \caption{\\$h_{\text{edge}}=0$ \\ $h_{\text{node}}=0$ \\ $h_{\text{norm}}=0$ \\ $h_{\text{den}}=0.33$}
         \label{fig:hete2}
     \end{subfigure}
     \hfill
     \begin{subfigure}[b]{0.24\columnwidth}
         \centering
         \includegraphics[width=0.8\textwidth]{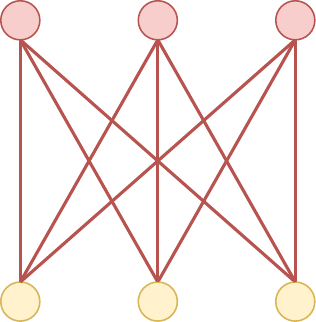}
         \caption{\\$h_{\text{edge}}=0$ \\ $h_{\text{node}}=0$ \\ $h_{\text{norm}}=0$ \\ $h_{\text{den}}=0$}
         \label{fig:hete3}
     \end{subfigure}
     \hfill
     \begin{subfigure}[b]{0.24\columnwidth}
         \centering
         \includegraphics[width=0.8\textwidth]{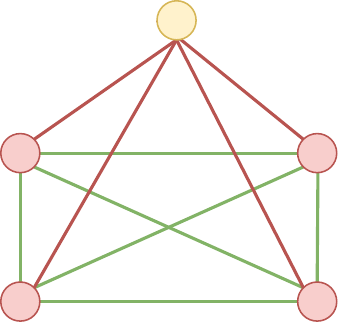}
         \caption{\\$h_{\text{edge}}=0.6$ \\ $h_{\text{node}}=0.6$ \\ $h_{\text{norm}}=0$ \\ $h_{\text{den}}=0$}
         \label{fig:imbalance1}
     \end{subfigure}
     \hfill
     \begin{subfigure}[b]{0.24\columnwidth}
         \centering
         \includegraphics[width=0.8\textwidth]{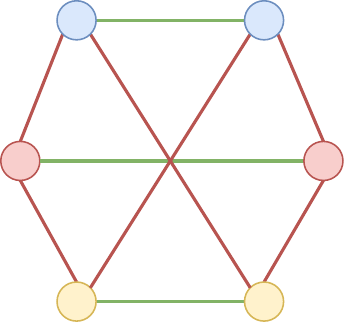}
         \caption{\\$h_{\text{edge}}=0.5$ \\ $h_{\text{node}}=0.5$ \\ $h_{\text{norm}}=0$ \\ $h_{\text{den}}=0.42$}
         \label{fig:regular1}
     \end{subfigure}
     \hfill
     \begin{subfigure}[b]{0.24\columnwidth}
         \centering
         \includegraphics[width=0.8\textwidth]{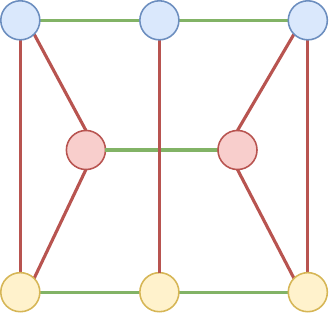}
         \caption{\\$h_{\text{edge}}=0.25$ \\ $h_{\text{node}}=0.42$ \\ $h_{\text{norm}}=0.11$ \\ $h_{\text{den}}=0.5$}
         \label{fig:regular2}
     \end{subfigure}
    \caption{Examples of graphs with different label-topology relationships and comparison of different homophily measures. The node colour represents the node labels. The red edges connect nodes of different labels, while the green edges connect nodes of the same labels. Figure \ref{fig:homo1} - \ref{fig:homo3} shows homophilic graphs of different densities. $h_{\text{den}}$ gives a higher score when a graph is dense, while the other metrics give the same scores. Figure \ref{fig:hete2} and \ref{fig:hete3} are two graphs that only consist of inter-class edges, but are of different densities. Figure \ref{fig:imbalance1} is a label-imbalanced graph. Figure \ref{fig:regular1} and Figure \ref{fig:regular2} are two regular graphs, where Figure \ref{fig:regular1} has an intra-class/inter-class edge ratio of $0.5$, Figure \ref{fig:regular2} is an example of an Erdos-Reyi graph sampled with uniform edge probability.}
    \label{fig:graph_examples}
\end{figure*}

Several methods have been proposed to measure the homophily of a graph~\citep{DBLP:conf/nips/ZhuYZHAK20,Lim2021-qz}. The two most used are edge homophily $h_{\text{edge}}$ and node homophily $h_{\text{node}}$: the former uses the proportion of edges connecting nodes of the same label
$h_{\text{edge}} = \frac{|{(u,v) \in E: y_u = y_v}|}{|E|}$,
while the later uses the proportion of a node's direct neighbours of the same label
$h_{\text{node}} = \frac{1}{N} \sum_{u \in V} \frac{|v \in \mathcal{N}_u:y_u = y_v|}{|\mathcal{N}_u|}$
where $\mathcal{N}_u$ is the neighbour set of node $u$, and $y$ is the node label.

As pointed out by \citet{Lim2021-qz}, both edge and node homophily suffer from sensitivity to both label number and label imbalance.
For instance, a balanced graph of $C$ labels would induce score of $\frac{1}{C}$ under both measures.
In addition, both metrics fail to handle cases of imbalanced labels, resulting in undesirably high homophily scores when the majority nodes are of the same label. 

To mitigate these issues, \citet{Lim2021-qz} proposed a new metric $h_{\text{norm}}$ that takes into account the label-wise node and edge proportion:
\begin{align}
    & h_{\text{norm}} = \frac{1}{K-1}\sum_{k=0}^{K-1}\left[h_k - \frac{|Y_k|}{N}\right]_+, \\
    \\
    & h_{k} = \frac{\sum_{u \in Y_k}{|v \in \mathcal{N}_u:y_u = y_v=k|}}{\sum_{u \in Y_k}|\mathcal{N}_u|},
\end{align}
$K$ is the number of unique labels, $Y_k$ is the node set of label $k$, and $h_k$ is the label-wise homophily.
Nevertheless, $h_{\text{norm}}$ only captures relative edge proportions and ignores graph connectivity, resulting in high homophily scores for highly disconnected graphs. 
For example, Figure \ref{fig:homo1} has the same $h_{\text{norm}}$ as \ref{fig:homo2} and \ref{fig:homo3}. The absence of edge density in the homophilic metric brings undesirable results in restructuring 
as the measurement always prefer a disconnected graph. 
Moreover, although $h_{\text{norm}}$ is lower-bounded by $0$, it does not explicitly define the meaning of $h_{\text{norm}}=0$ but instead refers to such graphs as less-homophilic in general, resulting in further confusion when comparing less-homophilic graphs.

Given the limitations of existing metrics, we propose a density-aware homophily metric. For a graph of $K>1$ labels, the following five propositions hold for our new metric:
\begin{enumerate}
    \item A dense homophilic graph of a complete set of intra-class edges and zero inter-class edges has a score of $1$. (Figure \ref{fig:homo3})
    \item A dense heterophilic graph of a complete set of inter-class edges and zero intra-class edges has a score of $0$. (Figure \ref{fig:hete3})
    \item An Erdos-Renyi random graph $G(n,p)$ of $n$ nodes and the edge inclusion probability $p$ has the score of $\approx 0.5$, i.e. a graph that is uniformly random (Figure \ref{fig:regular2}).
    \item A totally disconnected graph and a complete graph have the same score of $0.5$.
    \item For graphs with the same intra-class and inter-class edge ratios, the denser graph has a relatively higher score. (Figure \ref{fig:homo1}, \ref{fig:homo2} and \ref{fig:homo3}, Figure \ref{fig:hete2} and \ref{fig:hete3}, Figure \ref{fig:regular1} and \ref{fig:regular2})
\end{enumerate}

Propositions $1$ and $2$ define the limits of homophily and heterophily given a set of nodes and their labels. Propositions $3$ and $4$ define neutral graphs which are neither homophilic nor heterophilic. Proposition $3$ states that a uniformly random graph, which has no label preference on edges, is neutral thus has a homophily score of $0.5$. Proposition $5$ considers edge density: for graphs with the same tendencies of connecting inter- and intra-class nodes, the denser one has a higher absolute score value. The metric is defined as below.
\begin{align}
    \hat{h}_{\text{den}} = \min\{d_k - \bar{d}_k\}_{k=0}^{K-1}
    \label{eqn:h_hat_den}
\end{align}
where $d_k$ is the edge density of the subgraph formed by only nodes of label $k$, i.e. the intra-class edge density of $k$ (including self-loops), and $\bar{d}_{k}$ is the maximum inter-class edge density of label $k$
\begin{align}
    & d_k = \frac{2\left\lvert(u,v) \in E : k_u = k_v = k\right\rvert}{|Y_k|(|Y_k|+1)} \label{eqn:d_k},
\\
    & \bar{d}_{k} = \max\{d_{kj}:j=0,...,K-1; j\neq k\},
\end{align}
where $d_{kj}$ is the inter-class edge density of label $j$ and $k$, i.e. edge density of the subgraph formed by nodes of label $k$ or $j$.
\begin{align}
    d_{kj} = \frac{\big|(u,v) \in E : k_u = k, k_v = j\big|}{|Y_k||Y_j|}. \label{eqn:d_kj}
\end{align}

Equation \ref{eqn:h_hat_den} has the range $(-1,1)$. To make it comparable with the other homophily metrics, we scale it to the range $(0,1)$ using
\begin{align}
    h_{\text{den}} = \frac{1+\hat{h}_{\text{den}}}{2}.
\end{align}

\begin{table*}[t]
\centering
    \resizebox{0.85\linewidth}{!}{
          \begin{tabular}{l | l l  l  l  l l} 
             \toprule
              & \actor & \chameleon & \squirrel & \wisconsin & \cornell & \texas\\              
             \midrule
             GCN & $30.7\pm0.5$ & $59.8\pm2.6$ & $36.9\pm1.3$ & $64.1\pm6.3$ & $59.2\pm3.2$ & $64.1\pm4.9$\\
             GCN (GDC) & $35.0\pm0.5$ $(+4.3)$ & $62.2\pm1.2$ $(+2.4)$ & $45.3\pm1.3$ $(+8.4)$ & $53.9\pm2.6$ $(-10.2)$& $57.6\pm4.1$ $(-1.6)$ & $57.8\pm4.1$ $(-6.3)$\\
             GCN (ours) & $\underline{\textbf{36.2}}\pm1.0$ $(+5.5)$ & $\underline{\textbf{66.9}}\pm3.1$ $(+7.1)$ & $\underline{\textbf{55.7}}\pm2.4$ $(+18.8)$ & $\underline{\textbf{83.1}}\pm3.2$ $(+19.0)$& $\underline{\textbf{79.2}}\pm6.3$ $(+20.0)$ & $\underline{\textbf{78.4}}\pm5.4$ $(+14.3)$\\
             \midrule
             CHEV & $34.5\pm1.3$ & $66.0\pm2.3$ & $39.6\pm3.0$ & $82.5\pm2.8$ & $76.5\pm9.4$ & $79.7\pm5.0$\\
             CHEV(GDC) & $35.0\pm0.6$ $(+0.5)$ & $63.0\pm1.1$ $(-3.0)$ & $48.2\pm0.7$ $(+8.6)$& $83.5\pm2.9$ $(+1.0)$ &$\textbf{81.1}\pm3.2$ $(+4.6)$ & $79.2\pm3.0$ $(-0.5)$\\
             CHEV(Ours) & $\underline{\textbf{36.0}}\pm1.1$ $(+1.5)$ & $\textbf{66.8}\pm1.8$ $(+0.8)$ & $\underline{\textbf{55.0}}\pm2.0$ $(+15.4)$& $\textbf{84.3}\pm3.2$ $(+1.8)$& $80.8\pm4.1$ $(+4.3)$ & $\textbf{80.0}\pm4.8$ $(+0.3)$\\
             \midrule
             ARMA & $34.9\pm0.8$ & $62.1\pm3.6$ & $47.8\pm3.5$ & $78.4\pm4.6$ & $74.9\pm2.9$ & $82.2\pm5.1$\\ 
             ARMA (GDC) & $\underline{\textbf{35.9}}\pm0.5$ $(+1.0)$& $60.2\pm0.6$ $(-1.9)$& $47.8\pm0.8$ $(+0.0)$& $79.8\pm2.6$ $(+1.4)$& $78.4\pm4.1$ $(+3.5)$ & $78.4\pm3.2$ $(-3.8)$\\ 
             ARMA (ours) & $35.2\pm0.7$ $(+0.3)$& $\underline{\textbf{68.4}}\pm2.3$ $(+6.3)$& $\underline{\textbf{55.6}}\pm1.7$ $(+7.8)$& $\underline{\textbf{84.5}}\pm0.3$ $(+6.1)$& $\underline{\textbf{81.1}}\pm6.1$ $(+6.2)$ & $\textbf{81.1}\pm4.2$ $(-1.1)$\\ 
             \midrule
             GAT & $25.9\pm1.8$ & $54.7\pm2.0$  & $30.6\pm2.1$ & $62.0\pm5.2$ & $58.9\pm3.3$ & $60.0\pm5.7$\\
             GAT (GDC) & $35.0\pm0.6$ $(+9.1)$& $63.8\pm1.2$ $(+9.1)$& $48.6\pm2.1$ $(+18.0)$& $51.4\pm4.5$ $(-10.6)$& $58.9\pm2.2$ $(+0.0)$ & $77.1\pm8.3$ $(+17.1)$\\
             GAT (ours) & $\underline{\textbf{35.6}}\pm0.7$ $(+9.7)$& $\underline{\textbf{66.5}}\pm2.6$ $(+11.8)$& $\underline{\textbf{56.3}}\pm2.2$ $(+25.7)$& $\underline{\textbf{84.3}}\pm3.7$ $(+22.3)$& $\underline{\textbf{81.9}}\pm5.4$ $(+23.0)$ & $\underline{\textbf{79.8}}\pm4.3$ $(+19.8)$\\
             \midrule
             SGC & $28.7\pm1.2$ & $33.7\pm3.5$ & $46.9\pm1.7$& $51.8\pm5.9$& $58.1\pm4.6$ & $58.9\pm6.1$\\
             SGC (GDC) & $34.3\pm0.6$ $(+5.6)$& $60.6\pm1.5$ $(+26.9)$& $51.4\pm1.6$ $(+4.5)$& $53.7\pm5.1$ $(+1.9)$& $56.2\pm3.8$ $(-1.9)$ & $60.3\pm6.3$ $(+1.4)$\\
             SGC (ours) & $\underline{\textbf{34.9}}\pm0.7$ $(+6.2)$& $\underline{\textbf{67.1}}\pm2.9$ $(+33.4)$& $\underline{\textbf{52.3}}\pm2.3$ $(+5.4)$& $\underline{\textbf{77.8}}\pm4.7$ $(+26.0)$& $\underline{\textbf{73.5}}\pm4.3$ $(+15.4)$ & $\underline{\textbf{74.4}}\pm6.0$ $(+15.5)$\\
             \midrule
             APPNP & $35.0\pm1.4$ & $45.3\pm1.6$ &$31.0\pm1.6$ & $81.2\pm2.5$ & $70.3\pm9.3$ & $79.5\pm4.6$\\
             APPNP (GDC) & $35.7\pm0.5$ $(+0.7)$ & $52.3\pm1.4$ $(+7.0)$& $40.5\pm0.8$ $(+9.5)$& $80.2\pm2.4$ $(-1.0)$& $77.8\pm3.5$ $(+7.5)$ & $76.2\pm4.6$ $(-3.3)$\\
             APPNP (ours) & $\textbf{35.9}\pm1.1$ $(+0.9)$& $\underline{\textbf{66.7}}\pm2.7$ $(+21.4)$& $\underline{\textbf{55.9}}\pm2.9$ $(+24.9)$& $\textbf{84.3}\pm4.2$ $(+3.1)$& $\underline{\textbf{81.6}}\pm5.4$ $(+11.3)$ & $\textbf{80.3}\pm4.8$ $(+0.8)$\\
             \midrule
             \midrule
             GPRGNN & $33.4\pm1.4$ & $64.4\pm1.6$ & $41.9\pm2.2$ & $\textbf{85.5}\pm5.0$ & $79.5\pm7.0$ & $\textbf{84.6}\pm4.0$\\
             GPRGNN (GDC) & $\textbf{34.4}\pm1.0$ $(+1.0)$& $61.9\pm1.7$ $(-2.5)$& $39.2\pm1.5$ $(-1.7)$ & $85.1\pm5.0$ $(-0.4)$& $\textbf{82.4}\pm4.7$ $(+2.9)$& $80.8\pm4.9$ $(-3.8)$\\
             GPRGNN (ours) & $34.1\pm1.1$ $(+0.7)$& $\textbf{65.5}\pm2.2$ $(+1.1)$& $\underline{\textbf{47.1}}\pm2.4$ $(+5.2)$& $85.1\pm4.1$ $(-0.4)$& $80.3\pm6.3$ $(+0.8)$& $84.3\pm5.1$ $(-0.3)$\\
             \midrule
             \midrule
             Geom-GCN$^\dagger$ & $31.6$ & $60.9$ & $38.1$ & $64.1$ & $60.8$ & $67.6$\\
             $\text{H}_2\text{GCN}^*$ & $35.9\pm1.0$ & $59.4\pm2.0$ & $37.9\pm2.0$ & $86.7\pm4.7$ & $82.2\pm6.0$ & $84.9\pm6.8$\\
             BernNet & $35.1\pm0.6$ & $62.0\pm2.3^\triangledown$ & $52.6\pm1.7^\triangledown$ & $84.9\pm4.5^\triangledown$ & $80.3\pm5.4^\triangledown$ & $83.2\pm6.5^\triangledown$\\
             PPGNN & $31.4\pm0.8$ & $67.7\pm2.3^\triangledown$ & $56.9\pm1.2^\triangledown$ & $88.2\pm3.3^\triangledown$ & $82.4\pm4.3^\triangledown$ & $89.7\pm4.9^\triangledown$\\\hline
        \end{tabular}
        }
\caption{\label{tab:results}Node classification accuracy. Results marked with $\dagger$, $*$ and $\triangledown$ are obtained from~\citet{Pei2020,DBLP:conf/nips/ZhuYZHAK20,Lingam2021-lo} respectively. Statistically significant results are underlined based on paired T-test of $p<0.01$.} \end{table*}

Propositions $1$, $2$, $4$ and $5$ are easy to prove. We hereby give a brief description for Proposition $3$.

\begin{restatable}[]{lemma}{hdenerdos}
\label{lemma:h_den_erdos}
$\forall{K>1}$, $\mathbb{E}[h_{\text{den}}] = 0.5$ for the Erdos-Renyi random graph $G(n,p)$.
(A proof is given in Appendix.)
\end{restatable}

Figure \ref{fig:graph_examples} shows some example graphs with four different metrics. Compared with $h_{\text{node}}$ and $h_{\text{edge}}$, $h_{\text{den}}$ is not sensitive to the number of labels and label imbalance. Compared with $h_{\text{norm}}$, $h_{\text{den}}$ is able to detect neutral graphs. $h_{\text{den}}$ gives scores in the range $(0, 0.5)$ for graphs of low-homophily, allowing direct comparison between them. $h_{\text{den}}$ considers edge density and therefore is robust to disconnectivity.

\section{Empirical results}
\label{sec:emperiments}


\begin{figure}[ht]
\centering
\begin{subfigure}{0.48\columnwidth}
\includegraphics[clip,width=\textwidth]{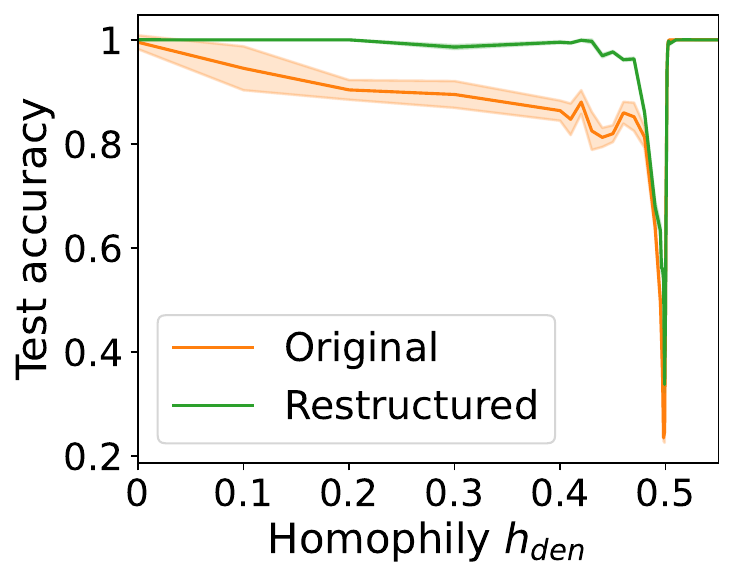}
\subcaption{GCN performance}
\label{fig:syn_cora_gcn}
\end{subfigure}%
\hfill
\begin{subfigure}{0.48\columnwidth}
\includegraphics[clip,width=\textwidth]{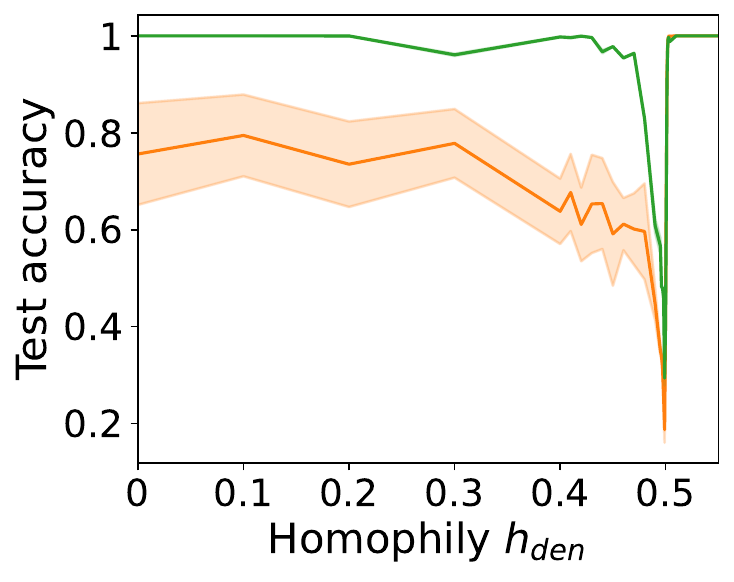}
\subcaption{SGC performance}
\label{fig:syn_cora_sgc}
\end{subfigure}%
\caption{Performance on synthetic datasets.}
\label{fig:syn_results}
\end{figure}

\paragraph{Datasets and models.}
We compare six classical GNNs: GCN~\citep{DBLP:conf/iclr/KipfW17}, SGC~\citep{DBLP:conf/icml/WuSZFYW19}, ChevNet~\citep{DBLP:conf/nips/DefferrardBV16}, ARMANet~\citep{DBLP:journals/corr/abs-1901-01343}, GAT~\citep{DBLP:conf/iclr/VelickovicCCRLB18}, and APPNP~\citep{DBLP:conf/iclr/KlicperaBG19}, on their performance before and after restructuring. We also report the performance using an additional restructuring methods: GDC~\citep{Klicpera2019-ut}. Five recent GNNs that target heterophilic graphs are also listed as baselines: GPRGNN~\citep{chien21}, H\textsubscript{2}GCN~\citep{DBLP:conf/nips/ZhuYZHAK20}, Geom-GCN~\citep{Pei2020}, BernNet~\citep{DBLP:conf/nips/HeWHX21} and PPGNN~\citep{Lingam2021-lo}. We run experiments on six real-world graphs: \texas{}, \cornell{}, \wisconsin{}, \actor{}, \chameleon{} and \squirrel{}~\citep{DBLP:journals/corr/abs-1909-13021, Pei2020}, as well as synthetic graphs of controlled homophily. Details of these datasets are given in Appendix.

\paragraph{Experimental setup.} 
Hyperparameters are tuned using grid search for all models on the unmodified and restructured graphs of each dataset. We record prediction accuracy on the test set averaged over 10 runs with different random initializations. We use the same split setting as~\citet{Pei2020, DBLP:conf/nips/ZhuYZHAK20}. The results are averaged over all splits. We adopt early stopping and record the results from the epoch with highest validation accuracy. We report the averaged accuracy as well as the standard deviation.
For the spectrum slicer in Equation \ref{eqn:slicer_matrix}, we use a set of 20 slicers with $s=40$ and $m=4$ so that the spectrum is sliced into 20 even range of $0.1$. In the restructure step, we add edges gradually and stop at the highest $h_{\text{den}}$ on validation sets before the validation homophily starts to decrease. All experiments are run on a single NVIDIA RTX A6000 48GB GPU unless otherwise noted.

\paragraph{Node classification results.}
Node classification tasks predict labels of nodes based on graph structure and node features. We aim to improve the prediction accuracy of GNN models by restructuring edges via the adaptive SC method, particularly for heterophilic graphs. The evaluation results are shown in Equation \ref{tab:results}. On average, the performance of GNN models is improved by 25\%. Training runtime for each dataset is reported in Appendix.

\begin{figure}[t]
\includegraphics[clip,width=\columnwidth]{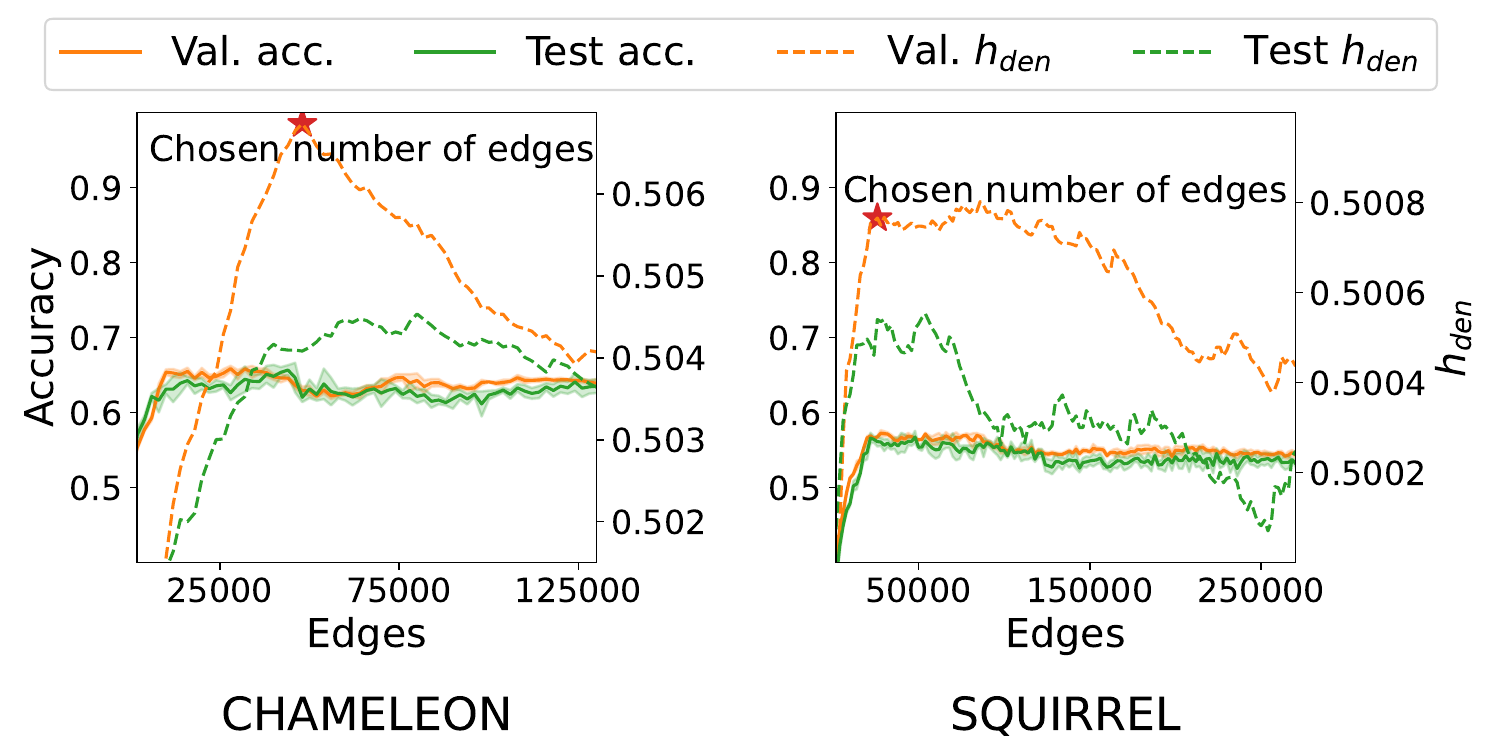}
\caption{Homophily and accuracy of GCN on validation and test sets as per edges numbers. The optimal number of edges are chosen based on $h_\text{den}$ on validation set.}
\label{fig:homo_vs_acc}\vspace{-0.3cm}
\end{figure}

\paragraph{Results on synthetic graphs}
\label{results:synthetic}
We run GCN and SGC on the synthetic dataset of controlled homophily range from $0$ to $1$. The model performance with homophily is plotted in Equation \ref{fig:syn_results}. As expected, higher homophily level corresponds to better performance for both GCN and SGC. All model reaches $100\%$ accuracy where homophily is larger than $0.55$. Performance reaches the lowest where homophily level is around $0.5$. This is where intra-class edge has same density as inter-class edges, hence hard for GNN models to capture useful structural patterns. Complete graphs, disconnected graphs and Erdos-renyi random graphs all fall into this category. When homophily level continues to decrease, performance starts to climb again. In fact, when homophily level reaches around $0$, GNN models are able to perform as good as when homophily level is high. An explanation is, in such low-homophily cases, a node $v$ aggregates features from nodes of other labels except its own, this aggregated information forms a useful pattern by itself for predicting the label of $v$, without needing information from nodes of the same label. Related observations are also reported by~\citet{Luan2021-xl, ma2022}. This shows traditional GNNs are actually able to capture the "opposite attract" pattern in extreme heterophiic graphs. In less-homophilic graphs where the homophily is less $0.5$, our adaptive spectral clustering method is able to lift homophily to a very high level, leading to boosted performance. We also note that the performance variance on the rewired graph is much lower than these on the original graph. We detail the synthetic data generation process in Appendix.

\paragraph{Restructuring based on homophily.}
As homophily and performance are correlated, in the restructuring process, number of edges are chosen based on homophily level on the validation set. As shown in Equation \ref{fig:homo_vs_acc}, we chose $48000$ edges for \chameleon{} and $26000$ edges for \squirrel{}, each corresponds to the first peak of homophily on their validation set. A limitation of this method is it relies on a balanced data split, i.e. if the homophily on validation set does not represent homophily on the whole dataset, the restructured graph may not yield a better homophily and performance.


\section{Related Work}
\smallskip
\noindent\textbf{GNNs for heterophilic graphs.~}
Early GNNs assume homophily implicitly. Such an inductive bias results in a degenerated performance on less-homophilic graphs~\citep{Lim2021-qz}. Recently, homophily is also shown to be an effective measure of a graph's robustness to both over-smoothing and adversarial attacks. Node representations converge to a stationary state of similar values (``over-smoothed'') as a GNN goes deeper. A graph of low homophily is also more prone to this issue as the stationary state is reached with fewer GNN layers~\citep{yan2021two}. For the same reason, homophilic graphs are more resilient to a graph injection attack than their heterophilic counterparts. Some techniques defend against such attacks by improving or retaining homophily under graph injection~\citep{Zhu2021-lf,chen2021}.
\citet{Pei2020} firstly draw attention to the limitation of GNN on less-homophilic graphs. 
Since then, various GNNs have been proposed to improve performance on these graphs. H\textsubscript{2}GCN~\citep{DBLP:conf/nips/ZhuYZHAK20} show that proper utilization of ego-embedding, higher-order neighbourhoods, and intermediate embeddings can improve results in heterophilic graphs. A recent scalable model, LINKX~\citep{Lim2021-qz}, shows separating feature and structural embedding improves performance. \citet{kim21howto} study this topic specifically for graph attention and finds improvements when an attention mechanism is chosen according to homophily and average degrees. \citet{chien21} propose to use a generalized PageRank method that learns the weights of a polynomial filter and show that the model can adapt to both homophilic and heterophilic graphs. Similarly, \citet{Li2021-uj} use learnable spectral filters for achieving an adaptive model on graphs of different homophilic levels. \citet{zhu2021graph} recently propose to incorporate a learnable compatibility matrix to handle heterophily of graphs. \citet{ma2022} and \citet{Luan2021-xl} found that when node neighbourhoods meet some condition, heterophily does not harm GNN performance. However, \citep{DBLP:conf/kdd/SureshBNLM21} shows real-world graphs often exhibit mixing patterns of feature and neighbourhoods proximity so such "good" heterophily is rare.

\smallskip
\noindent\textbf{Adaptive spectral clustering.~}
From the deep learning perspective, most previous studies about spectral clustering aim at clustering using learning approaches, or building relations between a supervised learning model and spectral clustering.
\citet{Law_2017-vb} and \citet{Bach_undated-ln} reveal that minimizing the loss of node similarity matrix can be seen as learning the leading eigenvector representations used for spectral clustering. \citet{Bianchi2019-qj} train cluster assignment using a similarity matrix and further use the learned clusters in a pooling layer to for GNNs. Instead of directly optimizing a similarity matrix, \citet{Tian2014-ft} adopt an unsupervised autoencoder design that uses node representation learned in hidden layers to perform K-means. 
\citet{Chowdhury2020-pe} show that Gromov-Wasserstein learning, an optimal transport-based clustering method, is related to a two-way spectral clustering. 

\smallskip
\noindent\textbf{Graph restructuring and rewiring.~} GDC~\citep{Klicpera2019-ut} is one of the first works propose to rewire edges in a graph. It uses diffusion kernels, such as heat kernel and personalized PageRank, to redirect messages passing beyond direct neighbours. \citet{Chamberlain2021-tw} and \citet{eliasof2021pde} extend the diffusion kernels to different classes of partial differential equations. \citet{topping2022} studies the ``over-squashing" issue on GNNs from a geometric perspective and alleviates the issue by rewiring graphs. In a slightly different setting where graph structures are not readily available, some try to construct graphs from scratch instead of modifying the existing edges. \citet{Fatemi2021-is} construct homophilic graphs via self-supervision on masked node features. \citet{Kalofolias2016-wy} learns a smooth graph by minimizing $tr(\mX^T\mL\mX)$. These works belong to the "learning graphs from data" family~\citep{Dong2019-cc} which is relevant but different to our work because \begin{enumerate*}[label=\emph{\alph*})]
  \item these methods infer graphs from data where no graph topology is readily available, while in our setting, the original graph is a key ingredient;
  \item as shown by~\citet{Fatemi2021-is}, without the initial graph input, the performance of these methods are not comparable to even a naive GCN;
  \item these methods are mostly used to solve graph generation problems instead of node- or link-level tasks.
\end{enumerate*}

\section{Conclusion}
We propose an approach to enhance GNN performance on less-homophilic graphs by restructuring the graph to maximize homophily. Our method is inspired and closely related to Spectral Clustering (SC). It extends SC beyond the leading eigenvalues and learns the frequencies that are best suited to cluster a graph. To achieve this, we use rectangular spectral filters expressed in the Neumann series to slice the graph spectrum into chunks that umbrellas small ranges of frequency. We also proposed a new homophily metric that is density-aware, and robust to label imbalance, hence a better homophily indicator for the purpose of graph restructuring. There are many promising extensions of this work, such as using it to guard against over-smoothing and adversarial attacks, by monitoring changes in homophily and adopting tactics to maintain it. We hereby leave these as future work.

\section*{Acknowledgement}
This work was partly supported by Institute of Information \& communications Technology Planning \& Evaluation (IITP) grant funded by the Korea government (MSIT) (No.2019-0-01906, Artificial Intelligence Graduate School Program(POSTECH)) and National Research Foundation of Korea (NRF) grant funded by the Korea government (MSIT) (NRF-2021R1C1C1011375). Dongwoo Kim is the corresponding author.

\newpage
\section*{Appendix}

\medskip

\subsection{A. Johnson-Lindenstrauss Theorem}
Below, we present the Johnson-Lindenstrauss
Theorem by \citet{Dasgupta1999-va}. For the proof, we refer the readers to the original paper of \citet{Dasgupta1999-va}.
\label{appendix:johnson_lindenstrauss}
\begin{theorem}[Johnson-Lindenstrauss
\label{thm:johnson-lindenstrauss}
Theorem]
For any $0<\epsilon<1$ and any integer $n$, let $k$ be a positive integer such that
\begin{align}
    k \leq 4(\epsilon^2/2 - e^3/3)^{-1}\ln{n}
\end{align}
Then for any set $S$ of $n$ points in $\mathbb{R}^d$, there is a map $f$: $\mathbb{R}^d \longrightarrow \mathbb{R}^k$ such that for all $u,v\in S$,
\begin{align}
\label{eqn:johnson-lindenstrauss}
    (1-\epsilon)\lVert u - v \rVert^2 \leq \lVert f(u) - f(v) \rVert^2 \leq (1+\epsilon)\lVert u - v \rVert^2.
\end{align}
Furthermore, this map can be found in randomized polynomial time.
\end{theorem}
 In this work, we adopt a side result of their proof, which shows the projection $f = \mR'x$ is an instance of such a mapping that satisfies \cref{eqn:johnson-lindenstrauss}, where $\mR'$ is a matrix of random Gaussian variables with zero mean. \citet{Achlioptas2003-cn} further extends the proof and shows that a random matrix drawn from $[1,0,-1]$, or $[1,-1]$ also satisfies the theorem. We leave these projections for future study.

\subsection{B. Proofs of Lemma 1 and Lemma 2}
\label{appendix:proofs}
We recall \cref{lemma:neumann}:

\newmannlemma*

\begin{proof}[]
\label{appendix:neumann_proof}
With a slight abuse of notation, we use $\lambda(*)={\lambda_1, \lambda_2,...}$ to denote eigenvalues of a matrix. The spectral radius of a matrix is the largest absolute value of its eigenvalues $\rho(*)=\max\left(|\lambda\left(*\right)|\right)$.
$\mL$ is the normalized Laplacian matrix of $\gG$, therefore $0\leq \lambda(\mL)\leq 2$. According to eigenvalue properties, we have 
$$-1 < \frac{\lambda(\mL) - a}{2+\hat{\epsilon}}<1,$$
thus
$$-1<\lambda\left(\frac{\mL-a\mI}{2+\hat{\epsilon}}\right)<1$$ 
because $\hat{\epsilon}>0$ and $a\in[0,2]$. Because the power of the eigenvalues of a matrix is the eigenvalues of the matrix power, i.e. $\lambda(*^{2m})=\lambda(*)^{2m}$, we have
$$0<\lambda\left(\left(\frac{\mL-a\mI}{2+\hat{\epsilon}}\right)^{2m}\right) <\lambda\left(\frac{\mL-a\mI}{2+\hat{\epsilon}}\right)< \frac{2}{2+\hat{\epsilon}}.$$
Therefore,
$$\frac{1}{s^{2m}}<\lambda\left( \left(\frac{\mL-a\mI}{2+\hat{\epsilon}}\right)^{2m} + \frac{\mI}{s^{2m}}\right) < \frac{2}{2+\hat{\epsilon}} + \frac{1}{s^{2m}}.$$ 
Hence, $\forall \hat{\epsilon} > \frac{2s^{2m}}{s^{2m}-1}-2$, 
$$0<\lambda\left(\left(\frac{\mL-a\mI}{2+\hat{\epsilon}}\right)^{2m} + \frac{\mI}{s^{2m}}\right)<1.$$
In another word, $\rho(\mI-\mT)<1$.
Gelfand's formula shows that if $\rho(\mI - \mT) < 1$, then $\lim_{p\to\infty}(\mI -\mT)^p=0$ and the inverse of $\mT$ can be expressed by a Neumann series $\mT^{-1} = \sum_{p=0}^\infty(\mI - \mT)^p$.
\end{proof}

We recall \cref{lemma:h_den_erdos}:
\hdenerdos*

\begin{proof}[]
\label{appendix:h_den_erdos_proof}
For each node label $k$ of $|Y_k|$ nodes, the are at most $\binom{|Y_k| + 1}{2} = |Y_k|(|Y_k|+1)/2$ intra-class edges (including self-loops). For each pair of label $(k,j)$, there are at most $|Y_k||Y_j|$ inter-class edges. On average $G(n,p)$ has $\binom{n+1}{2}p$ edges, among which $\binom{|Y_k| + 1}{2}p=\frac{|Y_k|(|Y_k|+1)p}{2}$ are intra-class for $k$, and $|Y_k||Y_j|p$ are inter-class for the class pair $(k,j)$. Hence from \cref{eqn:d_k} we have $\mathbb{E}[d_k] = p$ , from \cref{eqn:d_kj} we have $\mathbb{E}[d_{kj}] = p$. Substitute $d_k$ and $d_{kj}$ in \cref{eqn:h_hat_den} we have $\mathbb{E}[\hat{d}_{\text{den}}] = 0$ and $\mathbb{E}[d_{\text{den}}] = 0.5$.
\end{proof}

\subsection{C. Spectral Clustering}
\label{appendix:sc}
A simplified Spectral Clustering (SC) algorithm involves the following four steps:
\begin{enumerate}
    \item Perform eigendecomposition for the Laplacian matrix to obtain eigenvalues $(\lambda_1, \lambda_2, ..., \lambda_N)$ sorted in ascending order.
    \item Pick $L$ ($1 < L \leq N$) eigenvectors $\vu_1, ..., \vu_L$ associated with the leading $L$ eigenvalues.
    \item Represent a node $i$ with a vector $\vf_i$ whose element are from the chosen eigenvectors: $\vf_{i} =\left[u_1(i), u_2(i), ..., u_L(i)\right]^T\in \mathbb{R}^L$. 
    \item Perform K-means with a distance measurement, such as the Euclidean distance $||\vf_{i} - \vf_{j}||$ or dot product similarity $\vf_{i}^T \vf_{j}$, to partition the nodes into $K$ clusters. 
\end{enumerate}

\subsection{D. Graph Restructuring Algorithm}
\label{appendix:algorithm}
The proposed graph restructuring algorithm is illustrated in \cref{alg:reconstruction}.

\begin{algorithm}[ht]
   \caption{Graph Restructuring Algorithm}
   \label{alg:reconstruction}
\begin{algorithmic}
   \STATE {\bfseries Input:} Graph $G$, homophily metric $h$, number of random sample $P$, spectrum band length $s$, edge increment number $n$
   \STATE {\bfseries Output:} Restructured adjacency matrix $\hat{\mA}$ 
   \STATE
   \STATE Sample $\mR \sim \mathcal{N}(0, \frac{1}{P\mathbf{I}})$
   \FOR{each spectrum band $(s,a)$}
   \STATE $\bm{\Gamma}_{s,a} \leftarrow \hat{g}_{s,a}(\mL)(\mR \mathbin\frown \mX)$ \hfill $\triangleright$ \cref{eqn:filtered}
   \ENDFOR
   \STATE $\Theta \leftarrow \argmin_\Theta \mathcal{L}(\Theta)$    \hfill $\triangleright$ \cref{eqn:objective}
   \STATE $\mH \leftarrow \Theta(\bm{\Gamma})$
   \STATE Compute $\mD'$ where $\mD'_{ij}=\lVert H_{i\cdot} - H_{j\cdot}\rVert$
   \STATE $\pi \leftarrow $ sorted index of $\mD'$ in descending order
   \STATE \hfill $\triangleright$ only use the lower triangular part
   \STATE $\mA' \leftarrow \mathbf{0}$ \hfill  $\triangleright$ empty matrix
   \STATE $\lambda, \lambda^{\text{old}} \leftarrow 0.5$ 
   \WHILE{$\lambda \geq \lambda^{\text{old}}$}
   \STATE $\lambda^{\text{old}} \leftarrow \lambda$ \hfill  $\triangleright$ keep the old homophily score
   \STATE $(i_1,j_1), ..., (i_{n}, j_{n}) \leftarrow  \text{pop}(\pi, n)$
   \STATE \hfill  $\triangleright$ next $n$ edges from sorted index
   \STATE $\mA'_{i_1j_1} \leftarrow 1, ..., \mA'_{i_nj_n} \leftarrow 1$ \hfill  $\triangleright$ add new edge
   \STATE $\lambda \leftarrow h(\mA')$ \hfill  $\triangleright$ compute the new homophily score
   \ENDWHILE
   \STATE \textbf{return} $\mA' + \mA'^\top$ \hfill  $\triangleright$ make sure graph is undirected
\end{algorithmic}
\end{algorithm}



\subsection{E. Spectral Expressive Power}
\label{appendix:spectral_expressive}

In this section, we analyze the ability of the adaptive spectral clustering to learn specific frequency patterns. Being able to adjust to different frequency patterns, to an extent, demonstrates the express power of a model in the spectral domain. As pointed out by \citet{DBLP:conf/iclr/BalcilarRHGAH21}, the majority of GNNs are limited to only low-pass filters and thus have limited expressive power, while only a few are able to capture high-pass and band-pass patterns.

\begin{figure}[ht]
\centering
\begin{subfigure}{.45\columnwidth}
\includegraphics[clip,width=\textwidth]{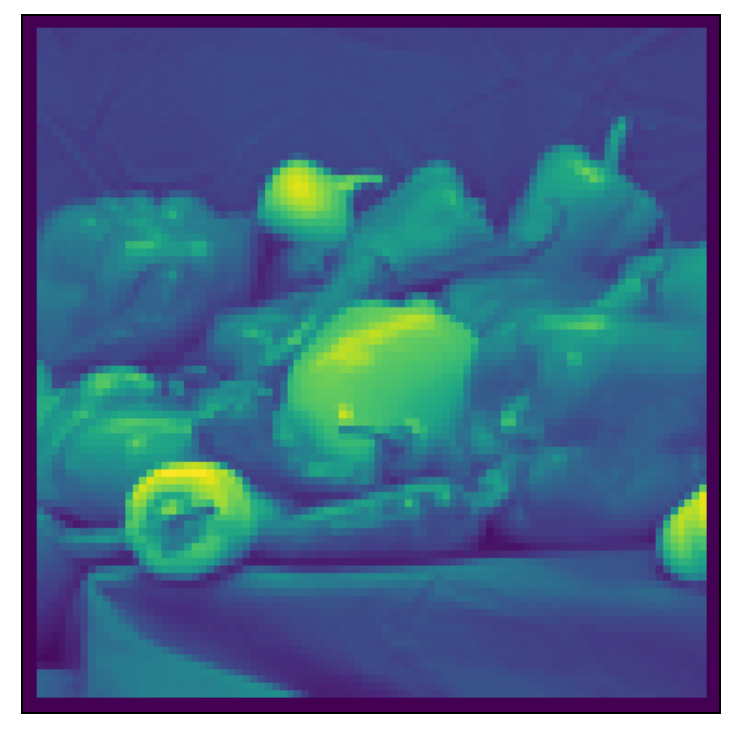}
\subcaption{Original image} 
\label{subfig:original}
\end{subfigure}
\begin{subfigure}{.45\columnwidth}
\includegraphics[clip,width=\textwidth]{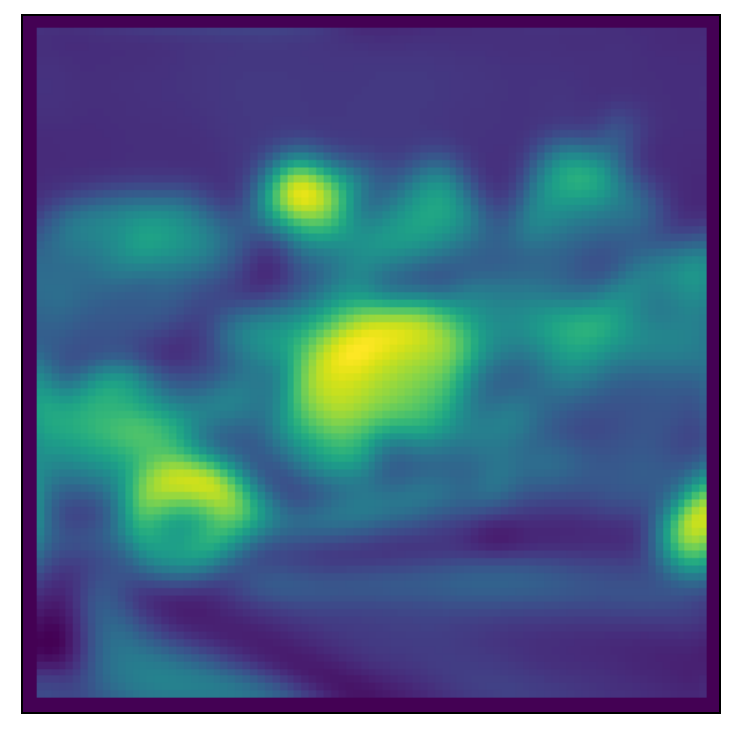}
\subcaption{Low-pass image} 
\label{subfig:low-pass}
\end{subfigure}
\\
\begin{subfigure}{.45\columnwidth}
\includegraphics[clip,width=\textwidth]{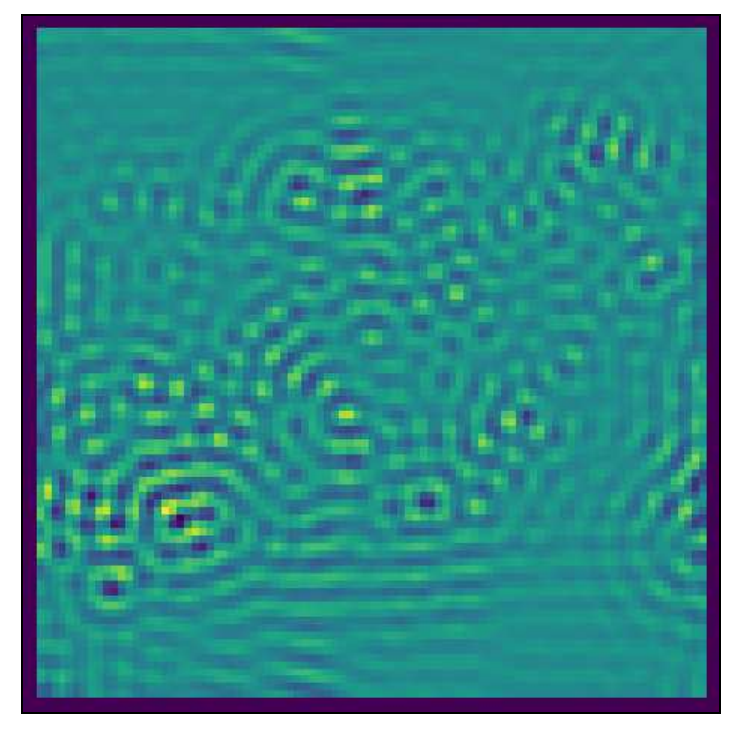}
\subcaption{Band-pass image}
\label{subfig:band-pass}
\end{subfigure}
\begin{subfigure}{.45\columnwidth}
\includegraphics[clip,width=\textwidth]{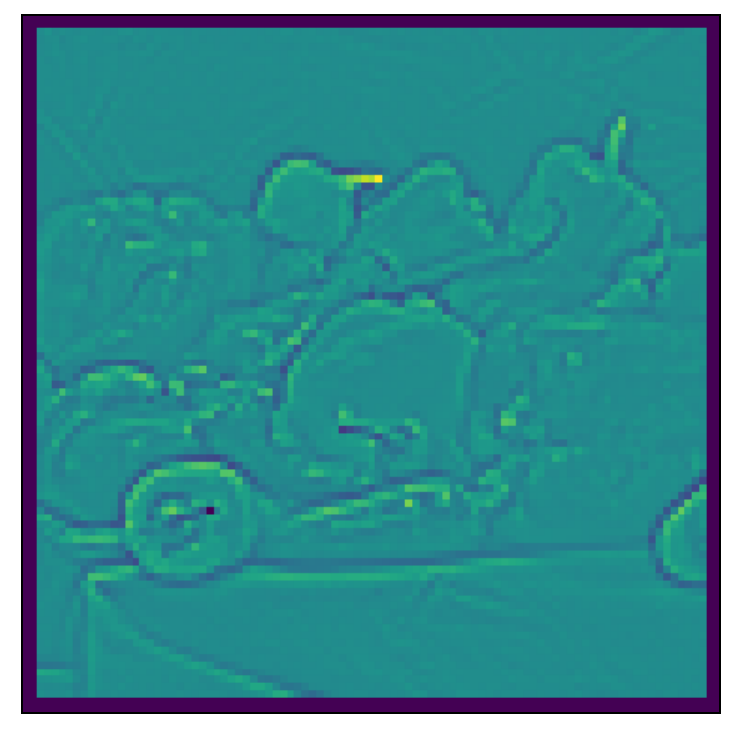}
\subcaption{High-pass image}
\label{subfig:high-pass}
\end{subfigure}
\caption{The original image and the filtered output using low-pass, band-pass and high-pass filters.}
\label{fig:img_filters}
\end{figure}

To evaluate this, we adopt the experimental setup of \citet{DBLP:conf/iclr/BalcilarRHGAH21} using filtered images. A real 100x100 image is filtered by three pre-defined low-pass, band-pass and high-pass filters: $\phi_1(\rho) = \exp{-100\rho^2}$, $\phi_2(\rho) = \exp{-1000(\rho-0.5)^2}$ and $\phi_3(\rho) = 1 - \exp{-10\rho^2}$, where $\rho=\rho_1^2 + \rho_2^2$ and $\rho_1$ and $\rho_2$ are the normalized frequencies in each direction of an image. The original image and the three filtered versions are shown in \cref{fig:img_filters}. The task is framed as a node regression problem, where we minimize the square error between $\mH\in\mathbb{R}^{N\times 1}$ in \cref{eqn:node_embedding} and the target pixel values, i.e.
\begin{align*}
\mathcal{L}'(\Theta) = \sum_{i=1}^N(\mH_i - \mY_i)^2
\end{align*}
where $\mY_i$ is the target pixel value of node $i$. We train the models with 3000 iterations and stop early if the loss is not improving in 100 consecutive epochs.

\begin{table}[ht]
\centering
    \resizebox{\columnwidth}{!}{
        \begin{tabular}{lcccccc}
        \toprule
        Task & MLP & GCN & GIN & GAT & ChevNet & Ours \\
        \midrule
        Low-pass & $43.42$ & $5.79$& $1.44$ & $2.30$ & $0.17$ & $0.07$\\
        Band-pass & $71.81$ & $74.31$ & $46.80$ & $74.04$ & $27.70$ & $2.94$\\
        High-pass & $19.95$ & $24.74$ & $17.80$ & $24.57$ & $2.16$ & $1.36$\\
        \bottomrule
        \end{tabular}
    }
\caption{Sum of squared errors}
\label{tab:square_loss}
\end{table}

\cref{tab:square_loss} shows the square loss of our method along MLP and 2 GNNs. Our method consistently outperforms other models. Some output images are shown in \cref{fig:filtered_img}. As expected, MLP fails to learn the frequency pattern across all three categories. GCN, GAT and GIN are able to learn the low-pass pattern but failed in learning the band and high-frequency patterns. Although ChevNet shows comparable results in the high-pass task, it is achieved with 41,537 trainable parameters while our method only requires 2,050 parameters. Lastly, our method is the only one that can learn and accurately resemble the band-pass image, demonstrating a better flexibility in learning frequency patterns.

\begin{figure}[ht]
\centering
\begin{subfigure}{.24\columnwidth}
\includegraphics[clip,width=\columnwidth]{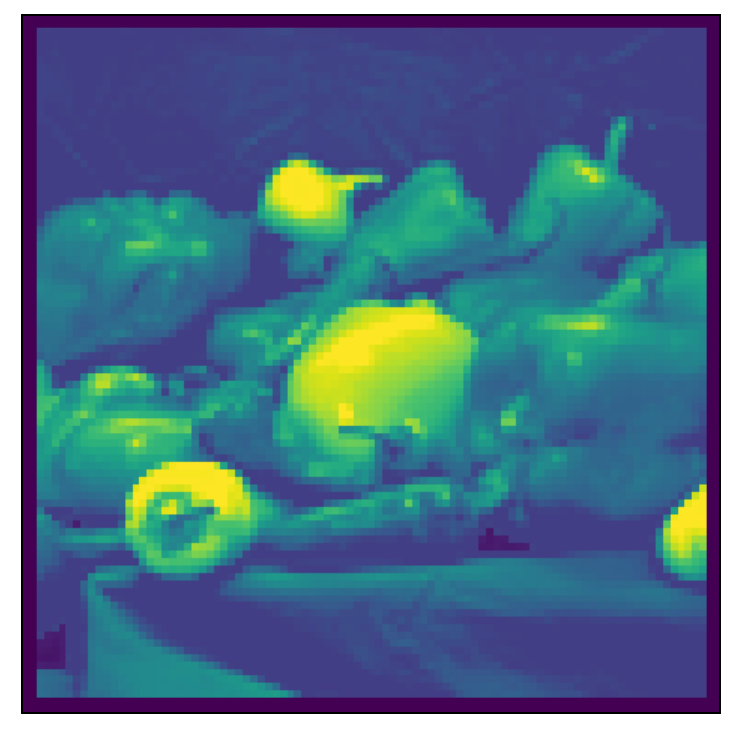}
\subcaption{Low-pass task: MLP} 
\label{subfig:low-pass-gcn}
\end{subfigure}
\hfill
\begin{subfigure}{.24\columnwidth}
\includegraphics[clip,width=\columnwidth]{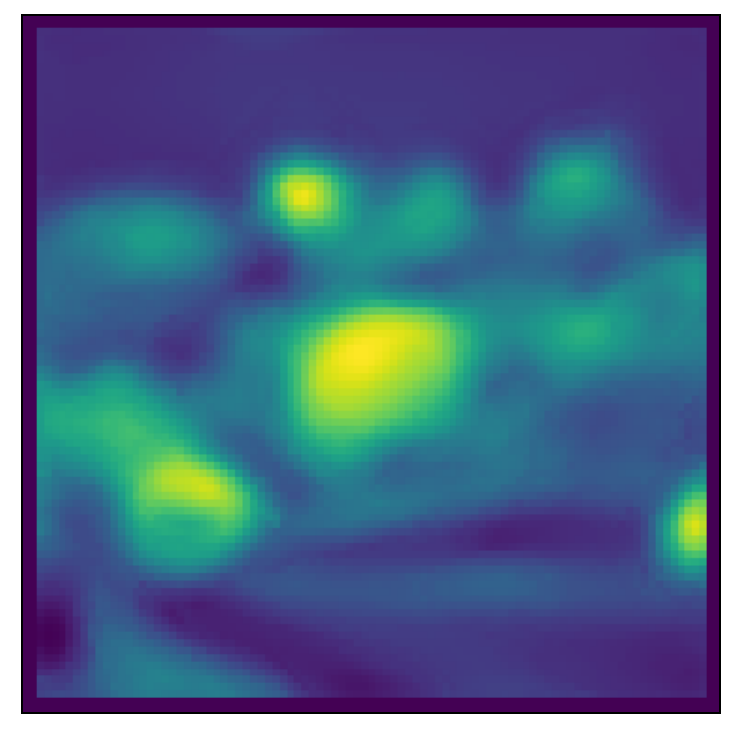}
\subcaption{Low-pass task: GIN} 
\label{subfig:low-pass-gin}
\end{subfigure}
\hfill
\begin{subfigure}{.24\columnwidth}
\includegraphics[clip,width=\columnwidth]{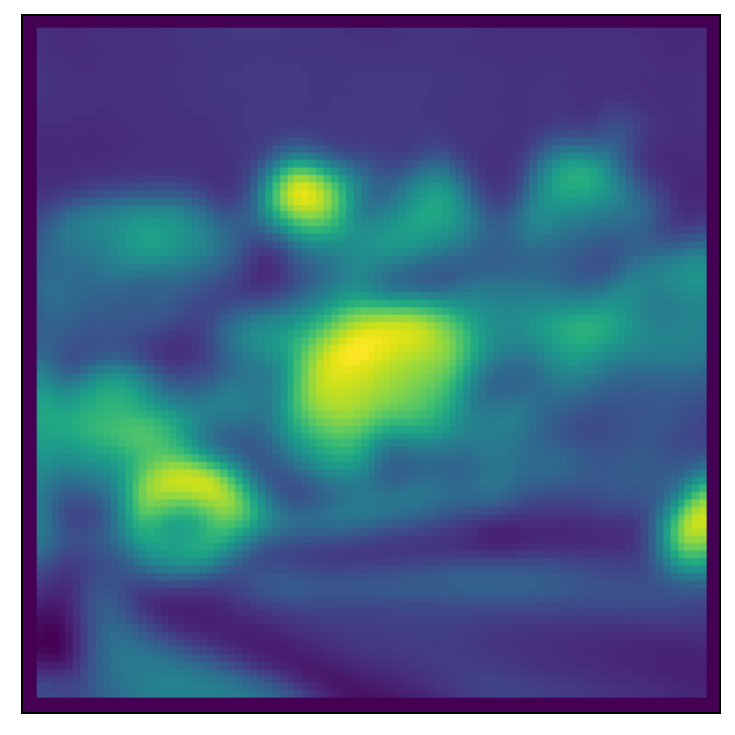}
\subcaption{Low-pass task: ChevNet} 
\label{subfig:low-pass-chev}
\end{subfigure}
\hfill
\begin{subfigure}{.24\columnwidth}
\includegraphics[clip,width=\columnwidth]{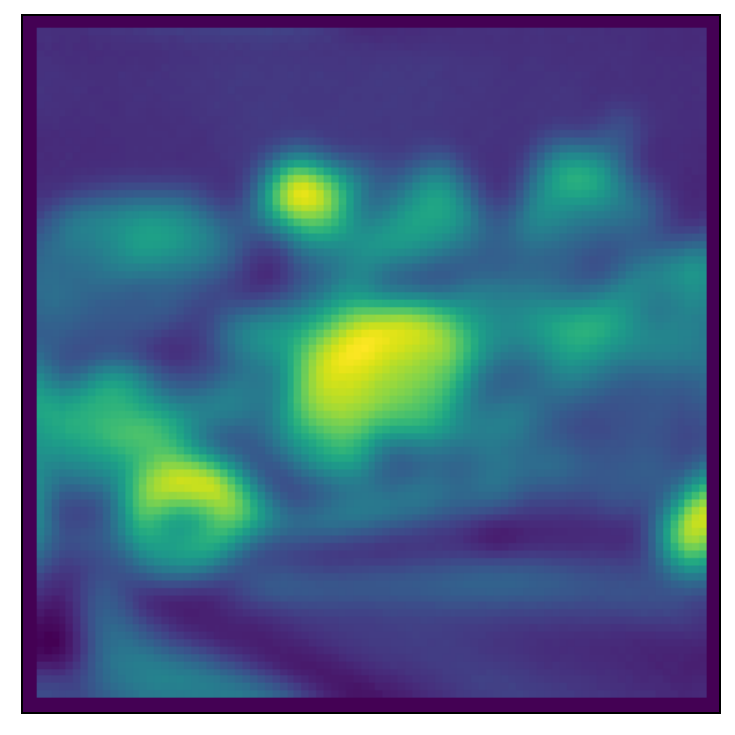}
\subcaption{Low-pass task: ours} 
\label{subfig:low-pass-slicers}
\end{subfigure}
\\
\begin{subfigure}{.24\columnwidth}
\includegraphics[clip,width=\columnwidth]{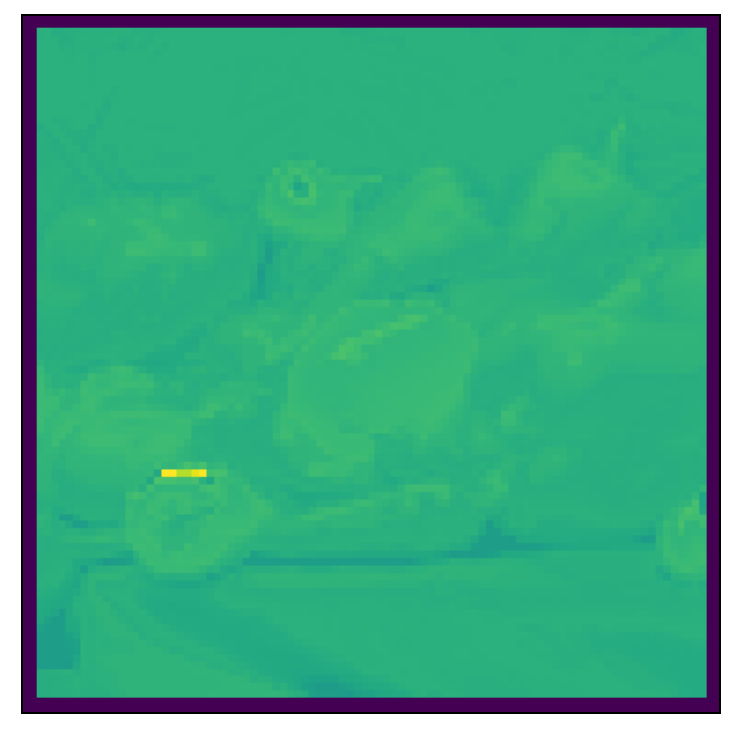}
\subcaption{Band-pass task: MLP} 
\label{subfig:band-pass-gcn}
\end{subfigure}
\hfill
\begin{subfigure}{.24\columnwidth}
\includegraphics[clip,width=\columnwidth]{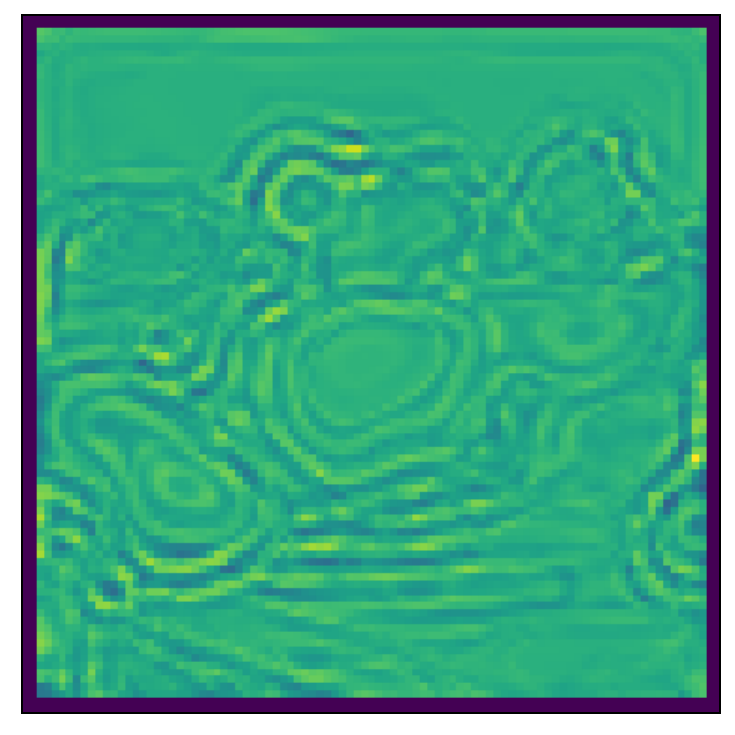}
\subcaption{Band-pass task: GIN} 
\label{subfig:band-pass-gin}
\end{subfigure}
\hfill
\begin{subfigure}{.24\columnwidth}
\includegraphics[clip,width=\columnwidth]{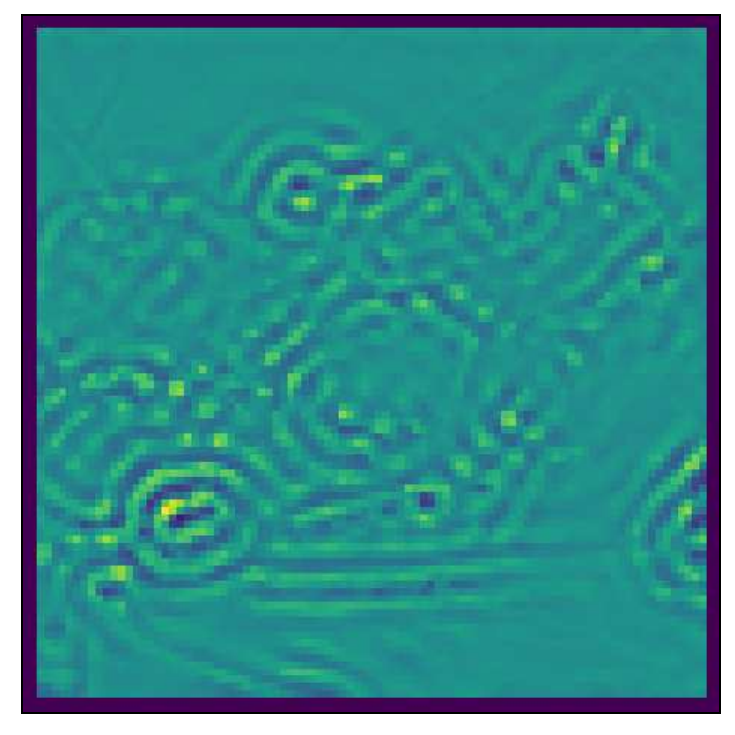}
\subcaption{Band-pass task: ChevNet} 
\label{subfig:band-pass-chev}
\end{subfigure}
\hfill
\begin{subfigure}{.24\columnwidth}
\includegraphics[clip,width=\columnwidth]{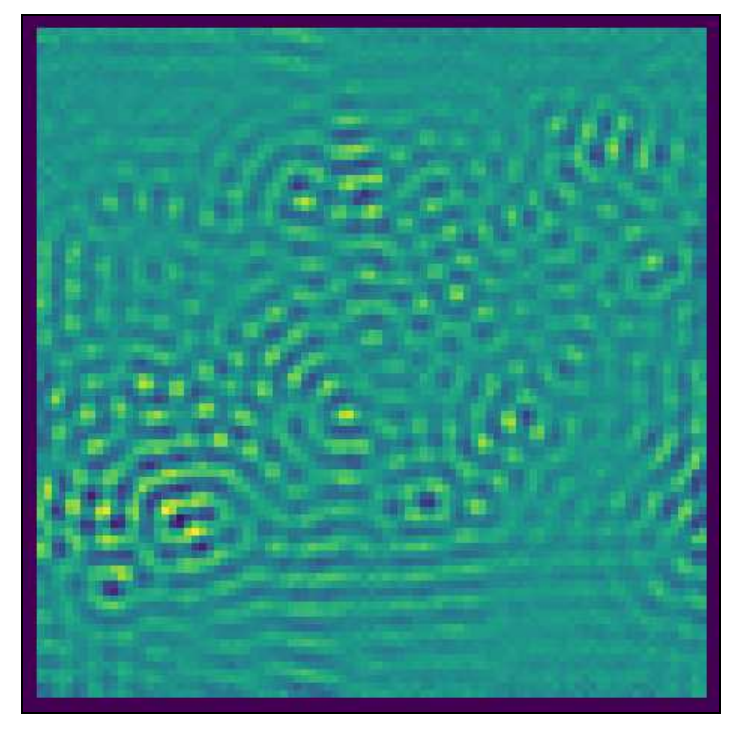}
\subcaption{Band-pass task: ours} 
\label{subfig:band-pass-slicers}
\end{subfigure}
\\
\begin{subfigure}{.24\columnwidth}
\includegraphics[clip,width=\columnwidth]{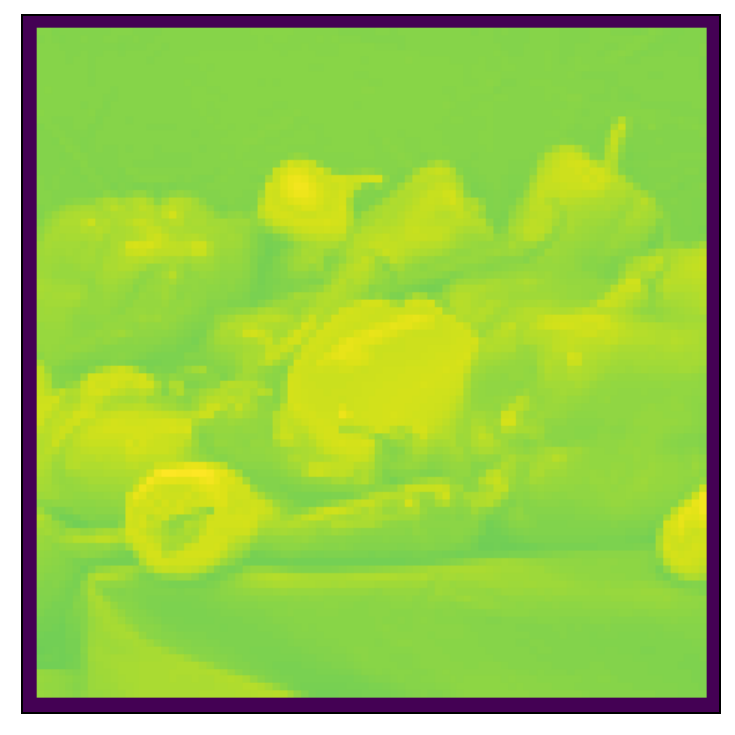}
\subcaption{High-pass task: MLP} 
\label{subfig:high-pass-gcn}
\end{subfigure}
\hfill
\begin{subfigure}{.24\columnwidth}
\includegraphics[clip,width=\columnwidth]{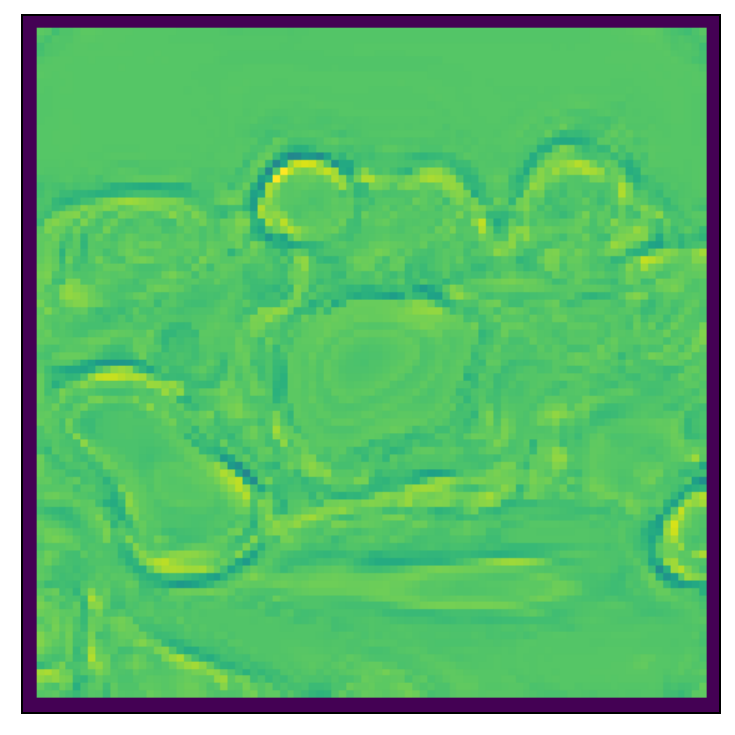}
\subcaption{High-pass task: GIN} 
\label{subfig:high-pass-gin}
\end{subfigure}
\hfill
\begin{subfigure}{.24\columnwidth}
\includegraphics[clip,width=\columnwidth]{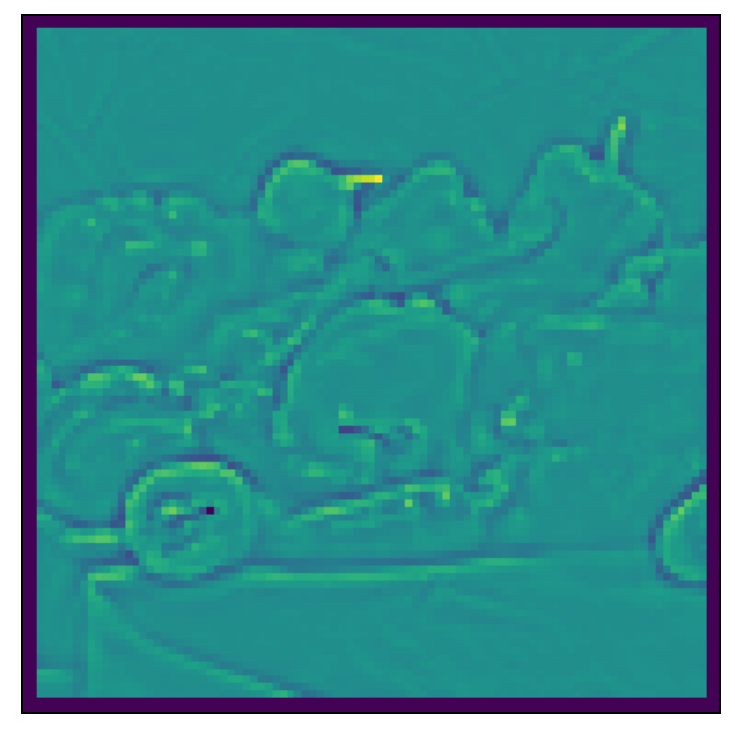}
\subcaption{High-pass task: ChevNet} 
\label{subfig:high-pass-chev}
\end{subfigure}
\hfill
\begin{subfigure}{.24\columnwidth}
\includegraphics[clip,width=\columnwidth]{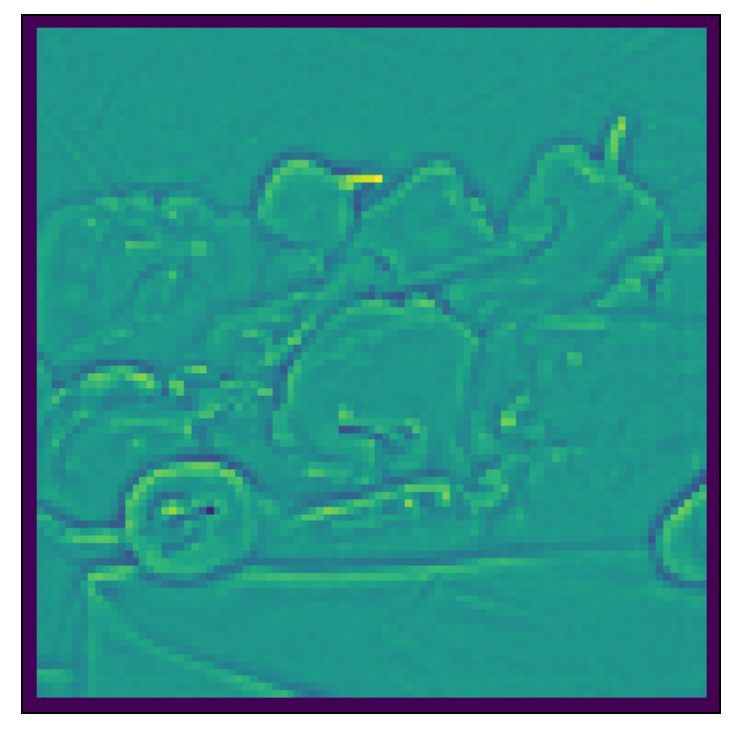}
\subcaption{High-pass task: ours} 
\label{subfig:high-pass-slicers}
\end{subfigure}
\caption{The low-pass, band-pass and high-pass images learned using MLP, GIN, ChevNet and our method. The images learned using our method better resemble the images shown in \cref{fig:img_filters} across all three categories.}
\label{fig:filtered_img}
\end{figure}

\subsection{F. Dataset Details}
\label{appendix:datasets}

We use 6 datasets listed in \cref{tab:dataset_detail}. \texas{}, \wisconsin{} and \cornell{} are graphs of web page links between universities, known as the CMU WebKB datasets. We use the pre-processed version in \citet{Pei2020}, where nodes are classify into 5 categories of course, faculty, student, project and staff. \squirrel{} and \chameleon{} are graphs of web pages in Wikipedia, originally collected by \citet{DBLP:journals/aim/SenNBGGE08}, then \citet{Pei2020} classifies nodes into 5 classes according to their average traffic. \actor{} is a graph of actor co-occurrence in films based on Wikipedia, modified by \citet{Pei2020} based on \citet{DBLP:conf/kdd/TangSWY09}.

\begin{table}[h]
\centering
\begin{tabular}{lcccc}
\toprule
Dataset & Classes & Nodes & Edges & Features \\
\midrule
\chameleon{}        & 5 & 2,277 & 36,101 & 2,325\\
\squirrel{}         & 5 & 5,201 & 217,073 & 2,089\\
\actor{}            & 5 & 7,600 & 26,752 & 931 \\
\texas{}            & 5 & 183 & 325 & 1,703 \\
\cornell{}          & 5 & 183 & 298 & 1,703  \\
\wisconsin{}        & 5 & 251 & 515 & 1,703     \\
\bottomrule
\end{tabular}
\caption{Dataset details}
\label{tab:dataset_detail}
\end{table}

\begin{table*}[h]
\centering
    \resizebox{0.65\textwidth}{!}{
          \begin{tabular}{l | l l  l  l  l l} 
             \toprule
              & \actor & \chameleon & \squirrel & \wisconsin & \cornell & \texas\\              
             \midrule
             GCN ($\mR\frown\mX$) & $\textbf{36.2}\pm1.0$& $\textbf{66.9}\pm3.1$& $\textbf{55.7}\pm2.4$& $\textbf{83.1}\pm3.2$& $\textbf{79.2}\pm6.3$& $\textbf{78.4}\pm5.4$\\
             GCN ($\mR$) & $36.1\pm0.9$& $53.5\pm6.1$& $47.4\pm4.1$& $78.4\pm4.2$& $71.1\pm6.0$& $67.8\pm8.3$\\
             \midrule
             CHEV($\mR\frown\mX$) & $\textbf{36.0}\pm1.1$& $66.8\pm1.8$& $\textbf{55.0}\pm2.0$& $84.3\pm3.2$& $80.8\pm4.1$& $80.0\pm4.8$\\
             CHEV($\mR$) & $36.0\pm1.2$& $53.8\pm6.8$& $46.2\pm4.3$& $82.4\pm5.1$& $72.3\pm4.7$& $77.3\pm4.5$\\
             \midrule
             ARMA ($\mR\frown\mX$) & $35.2\pm0.7$& $\textbf{68.4}\pm2.3$& $\textbf{55.6}\pm1.7$& $\textbf{84.5}\pm0.3$& $\textbf{81.1}\pm6.1$& $81.1\pm4.2$\\ 
             ARMA ($\mR$) & $35.1\pm0.9$& $53.2\pm6.4$& $47.2\pm4.5$& $77.8\pm6.5$& $66.8\pm6.5$& $73.5\pm7.2$\\ 
             \midrule
             GAT ($\mR\frown\mX$) & $\textbf{35.6}\pm0.7$& $\textbf{66.5}\pm2.6$& $\textbf{56.3}\pm2.2$& $\textbf{84.3}\pm3.7$& $\textbf{81.9}\pm5.4$& $\textbf{79.8}\pm4.3$\\
             GAT ($\mR$) & $35.1\pm0.9$& $53.2\pm7.0$& $46.5\pm4.5$& $80.6\pm3.7$& $71.6\pm7.6$& $66.8\pm7.9$\\
             \midrule
             SGC ($\mR\frown\mX$) & $\textbf{34.9}\pm0.7$& $\textbf{67.1}\pm2.9$& $\textbf{52.3}\pm2.3$& $\textbf{77.8}\pm4.7$& $\textbf{73.5}\pm4.3$& $\textbf{74.4}\pm6.0$\\
             SGC ($\mR$) & $35.0\pm0.6$& $52.6\pm6.7$& $43.9\pm5.1$& $76.1\pm5.0$& $68.1\pm7.2$& $66.2\pm7.2$\\
             \midrule
             APPNP ($\mR\frown\mX$) & $35.9\pm1.1$& $\textbf{66.7}\pm2.7$& $\textbf{55.9}\pm2.9$& $84.3\pm4.2$& $\textbf{81.6}\pm5.4$& $80.3\pm4.8$\\
             APPNP ($\mR$) & $36.0\pm3.7$& $53.8\pm6.5$& $48.2\pm4.7$& $82.0\pm4.4$& $72.7\pm6.6$& $76.2\pm4.4$\\
             \midrule
             GPRGNN ($\mR\frown\mX$) & $34.1\pm1.1$& $65.5\pm2.2$ & $\textbf{47.1}\pm2.4$ & $85.1\pm4.1$ & $80.3\pm6.3$ & $84.3\pm5.1$ \\
             GPRGNN ($\mR$) & $35.2\pm1.4$ & $53.5\pm5.7$ & $45.7\pm3.5$ & $81.2\pm4.3$ & $81.0\pm5.8$ & $74.3\pm6.8$ \\
             \midrule
        \end{tabular}
        }
\caption{\label{tab:ablation}Ablation study on model input.}
\end{table*}

\subsection{G. Hyper-parameter Details}
Hyper-parameter tuning is conducted separately using grid search for the graph structuring and node classification phases. In the graph restructuring phase, we fix $m=4$ in~\cref{eqn:slicer} and search for spectrum slicer width $s$ in \{10, 20, 40\}. $\epsilon$ in the loss function~\cref{eqn:objective} is searched in \{0.05, 0.1, 0.2\}, and the number of negative samples is searched in \{10, 20, 30, 64\}. The restructured graph with the best validation homophily is saved.
In the node classification phase, we use the same search range for common hyper-parameters: 0.1-0.9 for dropout rate, \{32, 64, 128, 256, 512, 1024\} for hidden layer dimension, \{1, 2, 3\} for the number of layers, and \{0, 5e-5, 5e-4, 5e-3\} for weight decay. We again use grid search for model-specific hyper-parameters. For CHEV, we search for the polynomial orders in \{1, 2, 3\}. For ARMA, we search for the number of parallel stacks in \{1, 2, 3, 4, 5\}. For GAT, we search for the number of attention heads in \{8, 10, 12\}. For APPNP, we search for teleport probability in 0.1-0.9.

We found the sensitivity to hyper-parameters varies. The spectral slicer width $s$, for example, yields similar homophily at 20 (slicer width = 0.2) and 40 (slicer width = 0.1), and starts to drop in performance when goes below 10 (slicer width = 0.4). This shows that, when the slicer is too wide, it fails to differentiate spectrum ranges that are of high impact, especially for datasets where high-impact spectrum patterns are band-passing (e.g. \squirrel). For the number of negative samples, smaller datasets tend to yield good performance on smaller numbers, while large datasets tend to yield more stable results (lower variance) on larger numbers. This is expected because, on the one hand, the total number of negative samples is much larger than that of positive samples, and a large negative sample set would introduce sample imbalance in small datasets. On the other hand, since samples are randomly selected, a large sample set tends to reduce variance in the training samples.

\subsection{H. Ablation Study}
\label{appendix:ablation_input}
On the one hand, in classical SC, clustering is normally performed based solely on graph structure, ignoring node features. Following this line, \citet{DBLP:conf/icassp/TremblayPBGV16} considers only random Gaussian signals when approximating SC under \cref{proposition:tremblay}. On the other hand, \citet{Bianchi2019-qj} adotps node feature for learning SC, with the motivation that "node features represent a good initialization". They also empirically show, on graph classification, that using only node feature or graph structure alone failed to yield good performance. To explore if the auxiliary information given by node features is indeed beneficial, we conduct ablation study and report the performance without including node features, denoted as $\mR$, in contrast to using both node features and graph structure, denoted as $\mR \frown \mX$. Results are shown in \cref{tab:ablation}. We note, except two small graphs \wisconsin{} and \actor{}, the node feature does contribute significantly to model performance.

\subsection{I. Training Runtime}
\label{appendix:runtime}

Apart from the theoretical analysis on time complexity in \cref{subsec:complexity}, we measure the model training time and report the average over 100 epochs in \cref{tab:runtime}. Results are reported in milliseconds.

\begin{table}[ht]
\centering
\resizebox{0.7\columnwidth}{!}{%
\begin{tabular}{l | c}
\toprule
                             & Runtime per epoch (ms) \\
\midrule
\actor                       & 109.85                 \\
\chameleon                   & 41.53                  \\
\squirrel                    & 87.95                  \\
\wisconsin                   & 4.52                   \\
\cornell                     & 4.88                   \\
\texas                       & 5.16                   \\
\bottomrule
\end{tabular}%
}
\caption{\label{tab:runtime}Training runtime.}
\end{table}

\subsection{J. Experiments on Synthetic Datasets}
\label{appendix:synthetic_datasets}
The correlation between homophily and GNN performance has been studied by~\citet{Zhu2021-lf, Luan2021-xl,ma2022}. In general higher homophily yields better prediction performance but, as shown by~\citet{Luan2021-xl, ma2022}, in some cases, heterophily can also contribute to better model performance. Therefore, it is worthwhile to investigate how the proposed density-aware homophily $h_{den}$ correlates to GNN performance.  Because the proposed homophily metric $h_{den}$ considers edge density as well as the proportion between intra-class and inter-class edges, existing synthetic data frameworks~\citep{DBLP:conf/nips/ZhuYZHAK20, Luan2021-xl,ma2022} cannot generate graphs with a target $h_{den}$. To address this problem, we design a synthetic graph generator that allows full control on homophily of the synthetic graphs.

\paragraph{Data generation process}
We generate synthetic graphs from a base graph where the node features are kept while all edges removed. As a result, the initial graph is totally disconnected. We then introduce initial edges by randomly connecting nodes until a minimum density threshold is met. Afterwards, edges are added one-by-one following the rule: if the current $h_{den}$ is lower than the target, a uniformly sampled intra-class edge will be added, and vice versa, until the target $h_{den}$ is satisfied. 
We generate synthetic graphs based on the \cora{}~\citep{DBLP:journals/aim/SenNBGGE08} dataset with varying $h_{den}$ from 0 to 1. We use the same node feature and node label as the original base graphs. 

\paragraph{Results on synthetic datasets}
Following~\citet{Zhu2021-lf}, we report the results for different homophily levels under the same set of hyperparameters for each model. We adopt the same training and early stopping configuration as reported in \cref{sec:emperiments}. Results are in \cref{fig:syn_results}. We also report the homophily score on the rewired graph for comparison in \cref{fig:syn_homo}. 

\begin{figure}[h!]
\centering
\includegraphics[clip,width=0.7\columnwidth]{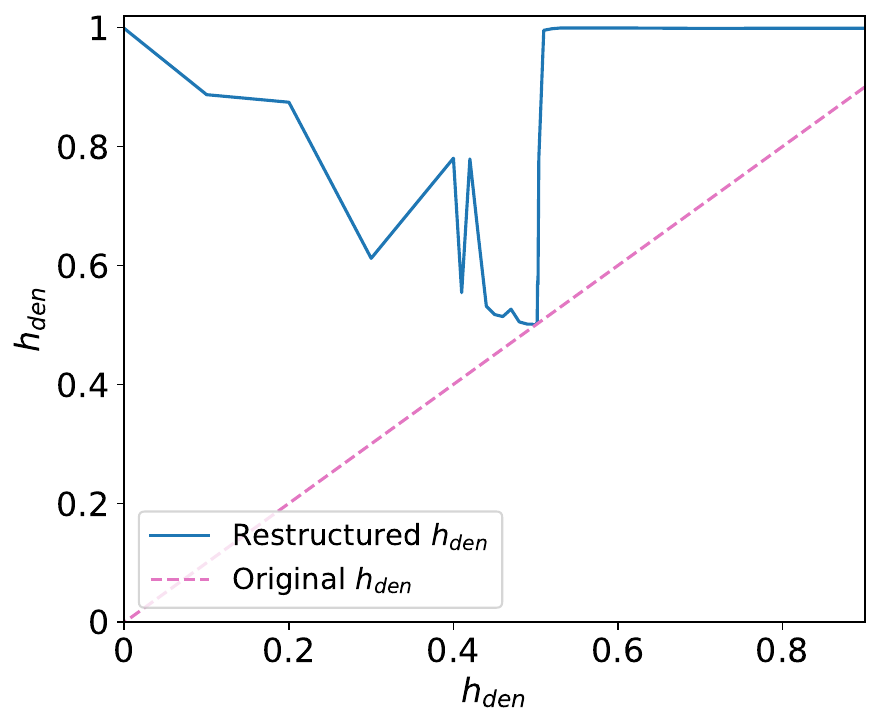}
\caption{Homophily on synthetic datasets}
\label{fig:syn_homo}\vspace{0.4cm}
\centering
\begin{subfigure}{.45\columnwidth}
\includegraphics[clip,width=\textwidth]{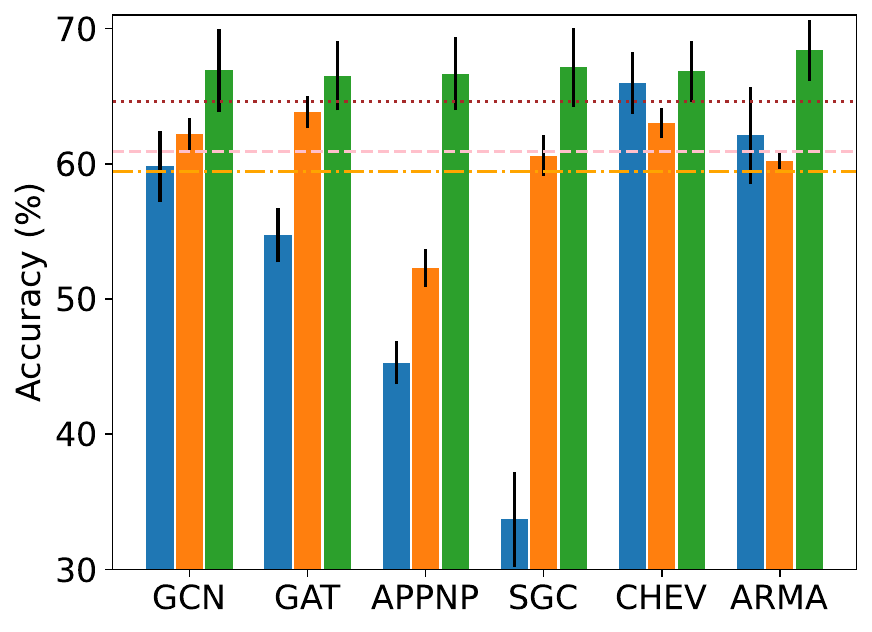}
\subcaption{\chameleon{}\label{subfig:hist_chame}}
\end{subfigure}%
\begin{subfigure}{.45\columnwidth}
\includegraphics[clip,width=\textwidth]{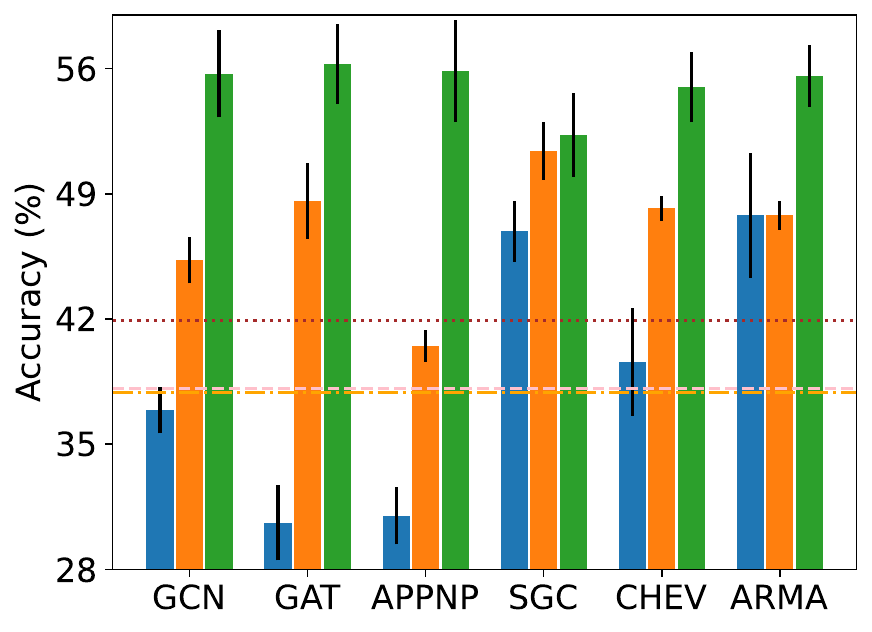}
\subcaption{\squirrel{}\label{subfig:hist_squir}}
\end{subfigure}%
\\
\begin{subfigure}{.45\columnwidth}
\includegraphics[clip,width=\textwidth]{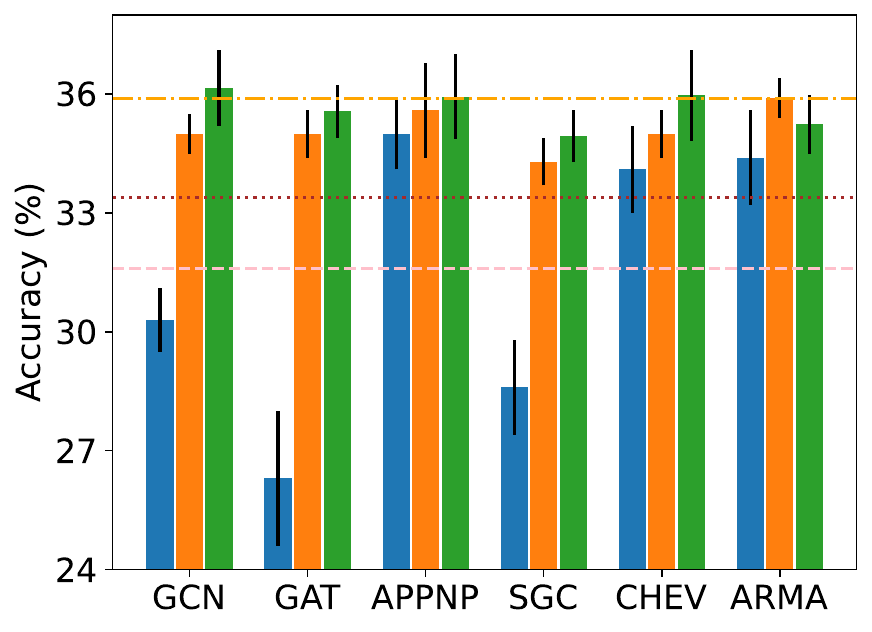}
\subcaption{\actor{}\label{subfig:hist_actor}}
\end{subfigure}%
\begin{subfigure}{.45\columnwidth}
\includegraphics[clip,width=\textwidth]{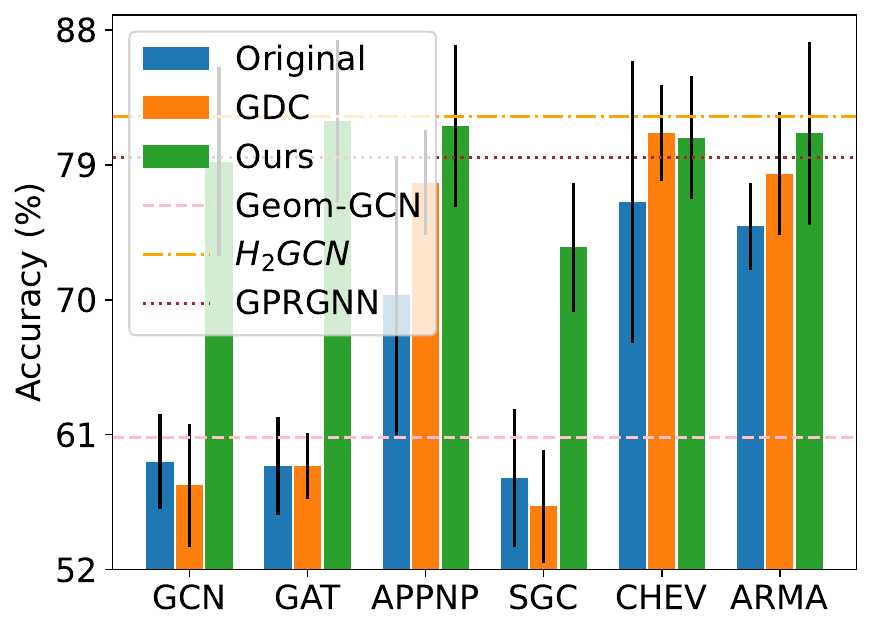}
\subcaption{\cornell{}\label{subfig:hist_cornell}}
\end{subfigure}%
\\
\begin{subfigure}{.45\columnwidth}
\includegraphics[clip,width=\textwidth]{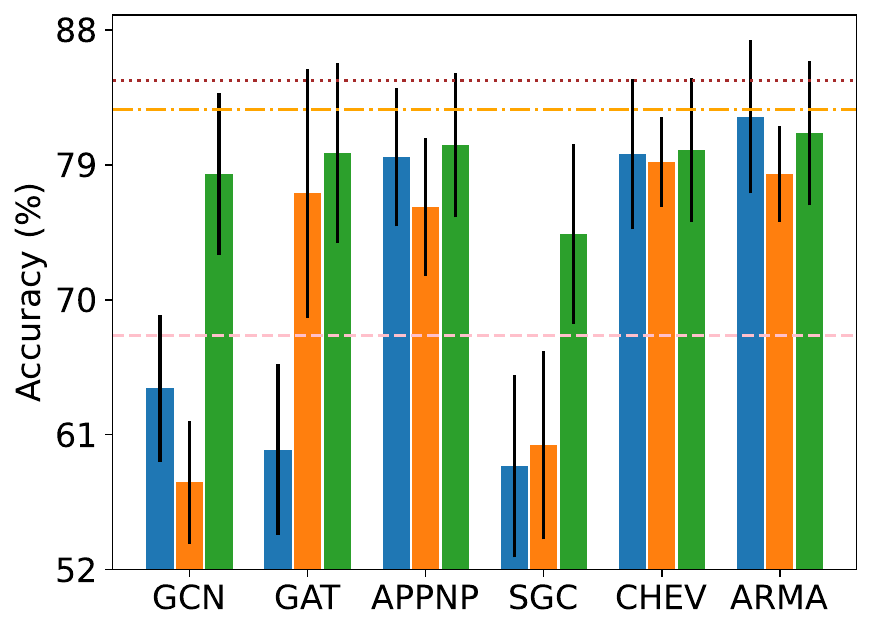}
\subcaption{\texas{}\label{subfig:hist_texas}}
\end{subfigure}%
\begin{subfigure}{.45\columnwidth}
\includegraphics[clip,width=\textwidth]{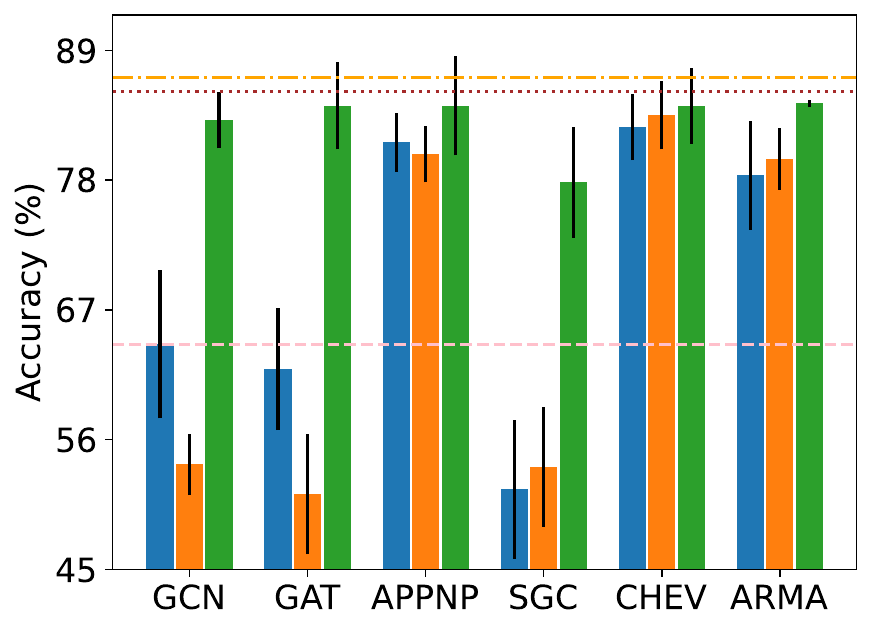}
\subcaption{\wisconsin{}\label{subfig:hist_wisconsin}}
\end{subfigure}%
\caption{Node classification accuracy of GNNs with and without restructuring on heterophilic graphs.}
\label{fig:histogram_hete}
\end{figure}

\subsection{K. Experiments on Node Classification}
\cref{fig:histogram_hete} shows the node classification accuracy from \cref{tab:results}. The performance of classical GNNs before and after restructuring are plotted as histograms. Performance of three GNNs that target less-homophilic datasets, Geom-GCN, $H_2GCN$ and GPRGNN, are plotted as horizontal lines.


\end{document}